%% file: main.tex
\documentclass{article}
\usepackage{comment}
\input{preamble}
\input{theoremstyles}

\input{macros}

\title{Hierarchical Domain Generalization}

\author{%
  Chenxiao Yang$^{1}$, Zhiyuan Li$^{1}$, Shai Ben-David$^{2}$, Nathan Srebro$^{1}$\\
  $^{1}$Toyota Technological Institute at Chicago\\
  $^{2}$University of Waterloo\\
  \texttt{\{chenxiao,zhiyuanli,nati\}@ttic.edu, shai@uwaterloo.ca}
}

\begin{document}

\maketitle

\begin{abstract}
We study hierarchical domain generalization as a problem of extrapolation from finite observed regions to an entire instance space, replacing i.i.d. sampling with arbitrary domain hierarchies. We show that the central obstruction is not only the complexity of the hypothesis class, but the train/test domain partition through which evidence is revealed. In particular, no matter how small the class or how large the training size, some partition makes generalization fail for some target. These results suggest that modern generalization theory must treat domain structure as a first-class object.
\end{abstract}

\section{Introduction}
\label{sec:introduction}

The concept of generalization in machine learning has been shaped by the seminal work of Vapnik and Chervonenkis on statistical generalization theory~\citep{vapnik1971uniform} and by Valiant's introduction of the PAC model for binary classification~\citep{valiant1984theory}.

In these setups, training data is drawn i.i.d.\ from some unknown data distribution $\cD$, and the learner's output is evaluated on the \emph{same} $\cD$. A successful learner is required to achieve low expected error on the distribution that generated the training data (or related data distributions in more recent work on domain adaptation learning, e.g., \cite{Ben-DavidBCP06}). This statistical framework has shaped the field's conception of ``what it means to generalize'' for four decades.

By contrast, the large-language-model era asks for a kind of generalization that this classical framework does not capture: trained on a very large fixed domain of inputs (e.g., the web), the learner should reason --- recover the underlying rule --- and produce correct outputs on relevant inputs outside its training domain.

However, a model trained on five-digit addition may fit its training data perfectly yet fail on six-digit numbers, as it often does in practice~\citep{anil2022exploring,lee2023teaching}. In that case, it has matched a pattern in the data, not the rule that adds arbitrarily long numbers. It fails to \emph{reason}. What the classical theory calls generalization is, in the sense most scientists and philosophers use the word, interpolation (biased or regulated by some prior knowledge). This is increasingly insufficient to capture many learning tasks of practical interest. Reasoning demands extrapolation, or out-of-domain generalization: from one kind of input to a different kind that shares some controlled (deterministic) similarity with the training data. 

We formalize this form of generalization as exact identification along a domain hierarchy. Instead of measuring expected error under a data-generating distribution, we require the learner to recover the target function on the entire instance space $\Omega$. Instead of assuming an i.i.d.\ observation model, we let the observable region grow through an externally specified hierarchy $\bar{\cX}=(\cX_n)_{n\geq 0}$ of finite cumulative domains, with $\emptyset=\cX_0\subsetneq\cX_1\subsetneq\cX_2\subsetneq\cdots$ and $\bigcup_{n\geq1}\cX_n=\Omega$ (so $\Omega$ is countable). At level $n$, the learner observes the labeled training set $\{(x,f(x)):x\in\cX_n\}$ and is evaluated on all of $\Omega$; equivalently, it must extrapolate from the observed region $\cX_n$ to the unobserved complement $\Omega\setminus\cX_n$. The only way to achieve this is to exactly \emph{identify} $f$.

This captures a common desideratum of modern reasoning systems: train on shorter inputs and generalize to longer ones~\citep{anil2022exploring,lee2023teaching,zhou2023algorithms} (a.k.a. length generalization), train on simpler instances and generalize to harder ones~\citep{bengio2009curriculum}, and train on bounded-horizon tasks and generalize to tasks with longer horizons~\citep{abbe2023generalization,yang2025recursive}. In the point-by-point case, where each increment is a singleton $\cX_n\setminus\cX_{n-1}=\{x_n\}$ and hence $\cX_n=\{x_1,\ldots,x_n\}$, the model recovers the presentation model of Gold's identification-in-the-limit paradigm~\citep{gold1967language}.

The central question is: given a hierarchy $\bar{\cX}$, how many levels $N$ must a learner observe to recover the target $f\in\cH$ on the entire domain $\Omega$? This question admits answers of varying strength: the required level may depend on both the target and the hierarchy, it may be common to all targets once the hierarchy is fixed, or it may be required to work uniformly over both targets and hierarchies.

The answer depends on how uniformly the level must be chosen. If the level may depend on both the target and the hierarchy, countability of $\cH$ is exactly the right condition. If the hierarchy is fixed but the same level must work for every target, the threshold becomes finiteness of $\cH$. But if the level must be chosen uniformly over the hierarchy itself, the situation changes completely: for every nontrivial hypothesis class, such a guarantee is possible only in the finite-domain case (\Cref{thm:A1-countable,thm:A2-finite,thm:A3-A4-characterization}).

Length generalization gives a concrete fixed-hierarchy instance of this question. The hierarchy is string length: training on all strings up to length $N$ and testing on longer strings asks whether that observed prefix determines the target rule. Since relevant classes such as finitely encoded Transformer families are infinite, we stratify them by complexity: for each bound $s$, we ask how many lengths are needed to identify every target of complexity at most $s$. Existing non-asymptotic length-generalization results estimate this sublevel domain complexity for particular representation classes and complexity scales~\citep{chen2025nonasymptotic,yang2026length}; qualitative and approximate variants fit the same fixed-hierarchy picture~\citep{huang2025formal,izzo2025quantitative}.

For a fixed hierarchy, such bounds have a clear interpretation: they say how far one must train along that hierarchy before the target is determined. But this still leaves open whether the bound reflects a property of the model class itself, or only of the particular hierarchy used to train and test. Our no-free-lunch result shows that the latter dependence is unavoidable. Even with the hypothesis class and complexity scale fixed, another hierarchy can make the finite-slice bounds grow arbitrarily fast (\Cref{thm:domain-generalization-nfl,cor:complexity-stratified-nfl}). Thus a length-generalization bound is a bound for the length hierarchy, not a hierarchy-free guarantee for arbitrary domain shifts.

Classical no-free-lunch theorems warn that an unrestricted hypothesis class admits no uniform learning guarantee~\citep{shalev2014understanding}. Here the obstruction comes from a second object that classical theory keeps offstage: the hierarchy that determines what is observed before what. A domain-generalization guarantee is therefore a property of the pair $(\cH, \bar{\cX})$, not of $\cH$ alone.

We further extend the framework and connect it to several neighboring learning settings, including online learning and generation in the limit~\citep{kleinberg2024languagegeneration}. Online learning counts mistakes rather than identifying levels, while generation in the limit asks for a weaker output than exact recovery. Furthermore, allowing error, or replacing hierarchies by distributional train/test sampling, does not by itself remove the difficulty: a guarantee still requires structural assumptions relating the observed region to the region on which the learner is evaluated. However, once we change the data model to positive-only presentations, uniform identification can hold for infinite classes and is controlled by overlap among positive regions.

\textbf{Organization.}~~
The rest of the paper is organized as follows. \Cref{sec:framework,sec:identification} introduce the hierarchy-based learning model and the main characterization theorem. \Cref{sec:complexity-stratified-domain-complexity,sec:length-gen} develop the fixed-hierarchy quantitative viewpoint, with length generalization as the central example. \Cref{sec:no-free-lunch-domain-generalization} shows why these bounds cannot be made independent of the hierarchy. \Cref{sec:extended-discussions} discusses extensions to related learning settings.

\input{sections/setup}
\input{sections/types}
\input{sections/characterization}
\input{sections/complexity_stratified}
\input{sections/length_generalization}
\input{sections/no_free_lunch}
\input{sections/extensions}
\section{Conclusion}
\label{sec:conclusion}

We characterized hierarchical domain generalization at four levels of uniformity, and proved a no-free-lunch theorem showing that hierarchy-uniform guarantees are achievable only in the trivial finite-domain case. For infinite classes, meaningful guarantees thus live at the level of fixed hierarchies and bounded-complexity slices, as in length generalization. The hierarchy is part of the learning problem.

\bibliography{iclr2026_conference}
\bibliographystyle{plainnat}

\clearpage
\appendix

\input{appendix/app_a_notation}
\input{appendix/app_b_characterization_proofs}
\input{appendix/app_c_complexity_stratified_proofs}
\input{appendix/app_d_no_free_lunch_proofs}
\input{appendix/app_e_online_learning}
\input{appendix/app_f_approximation}
\input{appendix/app_f_out_of_support}
\input{appendix/app_h_positive_only}
\input{appendix/app_g_limitations}

\clearpage

\end{document}

%% file: preamble.tex
\PassOptionsToPackage{table,dvipsnames}{xcolor}
\usepackage{xcolor}

\usepackage[preprint]{Styles/neurips_2026}

\usepackage{microtype}
\usepackage{graphicx}
\usepackage{array}
\usepackage{booktabs}
\usepackage{longtable}
\usepackage{tikz}
\usetikzlibrary{arrows.meta,positioning}
\usepackage{hyperref}

\input{math_commands.tex}

\usepackage{amsfonts}
\usepackage{amsmath,amsthm,amssymb,dsfont}
\usepackage{enumitem}
\usepackage{cleveref}
\usepackage{float}
\usepackage{makecell}
\usepackage[most]{tcolorbox}

\newtcolorbox{takeawaybox}{
  enhanced,
  breakable,
  colback=gray!3,
  colframe=black!55,
  boxrule=0.45pt,
  arc=1.2mm,
  left=7pt,
  right=7pt,
  top=6pt,
  bottom=6pt,
  before skip=0.8em,
  after skip=0.3em,
  fonttitle=\bfseries,
  coltitle=black,
  title=Takeaway,
  attach boxed title to top left={xshift=8pt,yshift=-2.1mm},
  boxed title style={
    colback=white,
    colframe=white,
    boxrule=0pt,
    left=2pt,
    right=2pt,
    top=0pt,
    bottom=0pt
  }
}

%% file: math_commands.tex
\usepackage{amsmath,amsfonts,bm}

\def\eqref#1{equation~\ref{#1}}

\def\1{\bm{1}}

\DeclareMathAlphabet{\mathsfit}{\encodingdefault}{\sfdefault}{m}{sl}
\SetMathAlphabet{\mathsfit}{bold}{\encodingdefault}{\sfdefault}{bx}{n}



%% file: theoremstyles.tex
\usepackage{aliascnt}

\newtheorem{theorem}{Theorem}

\newaliascnt{proposition}{theorem}
\newtheorem{proposition}[proposition]{Proposition}
\aliascntresetthe{proposition}
\crefname{proposition}{Proposition}{Propositions}
\Crefname{proposition}{Proposition}{Propositions}

\newaliascnt{lemma}{theorem}

\newtheorem*{lemma*}{Lemma}
\aliascntresetthe{lemma}
\crefname{lemma}{Lemma}{Lemmas}
\Crefname{lemma}{Lemma}{Lemmas}

\newaliascnt{corollary}{theorem}

\newtheorem{corollary}[corollary]{Corollary}
\aliascntresetthe{corollary}
\crefname{corollary}{Corollary}{Corollaries}
\Crefname{corollary}{Corollary}{Corollaries}

\newtheorem{definition}{Definition}
\newtheorem{example}{Example}

%% file: macros.tex
\newcommand{\Aone}{\textbf{A1}}    
\newcommand{\Atwo}{\textbf{A2}}    
\newcommand{\Athree}{\textbf{A3}}  
\newcommand{\Afour}{\textbf{A4}}   

\DeclareMathOperator{\Eq}{Eq}                

\DeclareMathOperator{\Ldim}{Ldim}

\newcommand{\cH}{\mathcal{H}}
\newcommand{\cD}{\mathcal{D}}
\newcommand{\cX}{\mathcal{X}}


%% file: sections/setup.tex
\section{Hierarchical Domain Generalization}
\label{sec:framework}

\subsection{Setup}
\label{sec:framework-setup}

Let $\Omega$ be a countably infinite instance space and let $\cH \subseteq \{0,1\}^{\Omega}$ be a binary hypothesis class.\footnote{Finite domains are used only as degenerate comparison cases, with the convention that the hierarchy collapses to $\cX_n=\Omega$ for all large enough $n$. The binary range is for concreteness; all characterizations extend to any finite label space.} We work in the realizable setting with deterministic labeling: an unknown target $f \in \cH$ generates labels $f(x) \in \{0,1\}$ for $x \in \Omega$. For any subset $\cX \subseteq \Omega$, write $f|_\cX$ for the labeled training set $\{(x, f(x)) : x \in \cX\}$.

To model out-of-domain generalization, we must specify how the observable part of $\Omega$ grows. The primitive object is not a sampling distribution, but a predetermined sequence of finite cumulative regions whose labels become available to the learner.

\begin{definition}[Domain hierarchy]
\label{def:main-domain-hierarchy}
A \emph{hierarchy} on $\Omega$ is a telescopic exhaustion $\bar{\cX}=(\cX_n)_{n\geq 0}$ by finite cumulative observed regions satisfying
\begin{equation}
  \emptyset=\cX_0 \;\subsetneq\; \cX_1 \;\subsetneq\; \cX_2 \;\subsetneq\; \cdots,
  \qquad
  |\cX_n|<\infty \text{ for every } n,
  \qquad
  \bigcup\nolimits_{n \geq 1}\cX_n = \Omega.
\end{equation}
\end{definition}
For $n\geq 1$, write $\Delta_n:=\cX_n\setminus\cX_{n-1}$ for the newly revealed increment at level $n$.
The definition is intentionally agnostic about how the regions are ordered. One common source is a sequence of externally specified domains, such as different populations, demographic subgroups, or intervention-based environments. Another source orders $\Omega$ by an intrinsic property of the inputs, such as length or difficulty, placing simpler inputs in early levels and harder ones later, as in length generalization and curriculum learning. When each increment is a singleton, say $\Delta_n=\{x_n\}$, equivalently $\cX_n=\{x_1,\ldots,x_n\}$, the hierarchy recovers Gold's identification-in-the-limit setting~\citep{gold1967language}.

\paragraph{Hierarchical Domain Generalization.}
Once a hierarchy is fixed, level $n$ splits $\Omega$ into the observed region $\cX_n$, from which the learner receives the labeled training set $f|_{\cX_n}$, and the unobserved complement $\Omega \setminus \cX_n$, which is out of domain relative to that level. The learner is evaluated on the entire $\Omega$, so success requires extrapolation from the observed region to the unseen part of the hierarchy. This in-domain/out-of-domain split is determined by $\bar{\cX}$ alone, without any distributional assumptions.

We idealize each observed level as fully labeled: at level $n$, the learner has access to all labels on $\cX_n$. Thus the question is not whether the learner can fit the observed region, but whether that region already determines the target beyond it.

\begin{definition}[Learner]
\label{def:main-learner}
A learning algorithm is a map $A: \bigcup_{\cX \subseteq \Omega,\ |\cX|<\infty}\{0,1\}^{\cX} \to \{0,1\}^{\Omega}$, which takes the labeled training set $f|_\cX$ as input and returns a binary predictor $A(f|_\cX)$. The learner may be chosen with knowledge of the hypothesis class $\cH$.
\end{definition}
Because the training set is indexed by a set rather than an ordered sequence, $A$ does not see the order in which sample points are revealed --- a difference from online learning, where the within-level order can itself carry information.

\begin{definition}[Domain complexity]
\label{def:main-domain-complexity}
Given a learner $A$, a target $f\in\cH$, and a hierarchy $\bar{\cX}$, the \emph{domain complexity} of $A$ on $f$ along $\bar{\cX}$ is
\begin{equation}
  N_A^{\bar{\cX}}(f)
  :=
  \min\big\{
    i\in\mathbb{N}_{\geq 1}:
    \forall n\geq i,\; A(f|_{\cX_n})=f
  \big\},
\end{equation}
with value $\infty$ if no such finite $i$ exists.
\end{definition}
Thus $N_A^{\bar{\cX}}(f)$ is the first level after which, for every $n\geq N_A^{\bar{\cX}}(f)$, training on the observed region $\cX_n$ yields a hypothesis correct on all of $\Omega$, equivalently on every later and currently unobserved part of the hierarchy.

%% file: sections/types.tex
\subsection{Levels of Learning Goals}
\label{sec:learning-goals}

We distinguish the possible forms of hierarchical domain generalization by the order of the quantifiers over the learner, the target, the hierarchy, and the identifying level. The first two criteria are hierarchy-dependent: in \Aone, the identifying level may depend on both the target and the hierarchy; in \Atwo, it may still depend on the hierarchy, but must be common to all targets in the class.

\begin{definition}[\Aone: \emph{Non-uniform identifiability}]
\label{def:A1}
$\cH$ is \emph{non-uniformly identifiable} (\Aone) if there exists a learning algorithm $A$ with the following property: for every target $f \in \cH$ and every hierarchy $\bar{\cX}$, there exists a finite $N$ such that $N_A^{\bar{\cX}}(f) \leq N$.
\end{definition}

\Atwo\ keeps the dependence on the hierarchy but requires one level to work for all targets.

\begin{definition}[\Atwo: \emph{Target-uniform identifiability}]
\label{def:A2}
$\cH$ is \emph{target-uniformly identifiable} (\Atwo) if there exists a learning algorithm $A$ with the following property: for every hierarchy $\bar{\cX}$, there is a finite $N$ such that, for every target $f \in \cH$, $N_A^{\bar{\cX}}(f) \leq N$.
\end{definition}

The remaining criteria are domain-uniform. Unlike \Aone\ and \Atwo, they require the level bound to be chosen without knowing the hierarchy, and hence without knowing the induced train/test split. This is the hierarchy analogue of choosing a PAC sample-size bound before the data-generating distribution is revealed.

\begin{definition}[\Athree: \emph{Domain-uniform identifiability}]
\label{def:A3}
$\cH$ is \emph{domain-uniformly identifiable} (\Athree) if there exists a learning algorithm $A$ with the following property: for every target $f\in\cH$, there exists a finite $N$ such that, for every hierarchy $\bar{\cX}$, $N_A^{\bar{\cX}}(f)\leq N$.
\end{definition}

\Athree\ still permits target-dependent bounds. \Afour\ removes this last dependence: it requires a single finite level that works for all targets and all hierarchies. In this sense, \Afour\ is the fully uniform analogue of a sample-complexity guarantee.

\begin{definition}[\Afour: \emph{Uniform identifiability}]
\label{def:A4}
$\cH$ is \emph{uniformly identifiable} (\Afour) if there exist a learning algorithm $A$ and a finite $N$ with the following property: for every target $f \in \cH$ and every hierarchy $\bar{\cX}$, $N_A^{\bar{\cX}}(f) \leq N$.
\end{definition}

%% file: sections/characterization.tex
\section{Exact Characterization of Domain Generalization}
\label{sec:characterization}
\label{sec:identification}

We now characterize the four learning goals from \Cref{sec:learning-goals}.

\begin{figure}[t]
\centering
\begin{tikzpicture}[
  font=\scriptsize,
  box/.style={draw, rounded corners=2pt, align=center, inner xsep=2.2pt, inner ysep=2.2pt, text width=3.55cm},
  aonebox/.style={box, fill=red!4},
  atwobox/.style={box, fill=green!5},
  athreebox/.style={box, fill=SkyBlue!14},
  condition/.style={box, fill=gray!8},
  formalarr/.style={-{Latex[length=1.7mm]}, thick, dashed, draw=gray!65},
  proofeq/.style={<->, >=Latex, thick, draw=black},
  proofback/.style={<->, >=Latex, thick, draw=black},
  refbadge/.style={font=\tiny, fill=white, inner sep=0.6pt, text=black},
  arrref/.style={font=\tiny, fill=white, inner sep=0.7pt, text=black}
]
\node[aonebox] (a1) {
  \textbf{\Aone: non-uniform}\\[-0.15ex]
  \resizebox{3.32cm}{!}{$\exists A\,\forall f\,\forall\bar{\cX}\,\exists N:\;N_A^{\bar{\cX}}(f)\leq N$}
};
\node[atwobox, right=0.72cm of a1] (a2) {
  \textbf{\Atwo: target-uniform}\\[-0.15ex]
  \resizebox{3.32cm}{!}{$\exists A\,\forall\bar{\cX}\,\exists N\,\forall f:\;N_A^{\bar{\cX}}(f)\leq N$}
};
\node[athreebox, right=0.72cm of a2] (a4) {
  \textbf{\Afour: uniform}\\[-0.15ex]
  \resizebox{3.32cm}{!}{$\exists A\,\exists N\,\forall f\,\forall\bar{\cX}:\;N_A^{\bar{\cX}}(f)\leq N$}
};

\node[aonebox, below=0.36cm of a1] (a1fixed) {
  \textbf{fixed hierarchy version}\\[-0.15ex]
  \resizebox{3.32cm}{!}{$\exists A\,\exists\bar{\cX}\,\forall f\,\exists N:\;N_A^{\bar{\cX}}(f)\leq N$}
};
\node[atwobox, below=0.36cm of a2] (a2fixed) {
  \textbf{fixed hierarchy common bound}\\[-0.15ex]
  \resizebox{3.32cm}{!}{$\exists A\,\exists\bar{\cX}\,\exists N\,\forall f:\;N_A^{\bar{\cX}}(f)\leq N$}
};
\node[athreebox, below=0.36cm of a4] (a3) {
  \textbf{\Athree: domain-uniform}\\[-0.15ex]
  \resizebox{3.32cm}{!}{$\exists A\,\forall f\,\exists N\,\forall\bar{\cX}:\;N_A^{\bar{\cX}}(f)\leq N$}
};

\node[aonebox, below=0.36cm of a1fixed] (a1target) {
  \textbf{target-chosen hierarchy version}\\[-0.15ex]
  \resizebox{3.32cm}{!}{$\exists A\,\forall f\,\exists\bar{\cX}\,\exists N:\;N_A^{\bar{\cX}}(f)\leq N$}
};
\node[condition, below=0.58cm of a1target] (countable) {
  \textbf{structural condition}\\[-0.15ex]
  $\cH$ countable
};
\node[condition] (finite) at (a2 |- countable) {
  \textbf{structural condition}\\[-0.15ex]
  $\cH$ finite
};
\node[condition] (omega) at (a4 |- countable) {
  \textbf{trivial case}\\[-0.15ex]
  $|\Omega|<\infty$
};

\node[refbadge, anchor=north east, yshift=-1pt] at (a1.south east) {\Cref{def:A1}};
\node[refbadge, anchor=north east, yshift=-1pt] at (a2.south east) {\Cref{def:A2}};
\node[refbadge, anchor=north east, yshift=-1pt] at (a3.south east) {\Cref{def:A3}};
\node[refbadge, anchor=north east, yshift=-1pt] at (a4.south east) {\Cref{def:A4}};

\draw[formalarr] (a1) -- (a1fixed);
\draw[formalarr] (a1fixed) -- (a1target);
\draw[formalarr] (a2) -- (a2fixed);
\draw[formalarr] (a4) -- (a3);

\draw[proofeq] (a1target) -- (countable);
\draw[proofeq] (a2fixed) -- (finite);
\draw[proofeq] (a3) -- (omega);
\node[arrref, anchor=west] at ([xshift=1pt,yshift=0.28cm]countable.north) {\Cref{thm:A1-equivalent-conditions}};
\node[arrref, anchor=west] at ([xshift=1pt,yshift=0.28cm]finite.north) {\Cref{thm:target-uniform-equivalences}};
\node[arrref, anchor=east] at ([xshift=-1pt,yshift=0.28cm]omega.north) {\Cref{thm:A3-A4-characterization}};
\draw[proofback, rounded corners=2pt] ([xshift=-4pt]countable.west) -- ++(-0.38,0) |- ([xshift=-4pt]a1.west);
\draw[proofback, rounded corners=2pt]
  (finite.south)
  -- ([yshift=-0.32cm]finite.south)
  -- ([xshift=2.35cm,yshift=-0.32cm]finite.south)
  -- ([xshift=2.35cm,yshift=0.42cm]a2.north)
  -- ([yshift=0.42cm]a2.north)
  -- (a2.north);
\draw[proofback, rounded corners=2pt] ([xshift=4pt]omega.east) -- ++(0.36,0) |- ([xshift=4pt]a4.east);
\node[arrref, overlay] at ([xshift=-0.60cm]a1target.west) {\Cref{thm:A1-countable}};
\node[arrref, overlay] at ([xshift=2.35cm]finite |- a1target) {\Cref{thm:A2-finite}};
\node[arrref, overlay] at ([xshift=0.50cm]a4.east |- a1target) {\Cref{thm:A3-A4-characterization}};

\draw[formalarr] (a4.west) -- (a2.east);
\draw[formalarr] (a2.west) -- (a1.east);
\draw[formalarr] (a2fixed.west) -- (a1fixed.east);
\draw[formalarr, rounded corners=2pt] (a3.east) -- ++(0.35,0) -- ++(0,1.65) -| (a1.north);
\draw[formalarr] (omega.west) -- (finite.east);
\draw[formalarr] (finite.west) -- (countable.east);
\end{tikzpicture}
\caption{\textbf{Key characterization.} Exact characterization of when a hypothesis class allows hierarchical domain generalization at levels that may depend on the target and hierarchy (\Aone), depend on the hierarchy but are uniform over targets (\Atwo), depend on the target but are uniform over hierarchies (\Athree), or are uniform over both targets and hierarchies (\Afour). Throughout the figure, $f\in\cH$, $\bar{\cX}$ ranges over hierarchies, and $|\cH|>2$. Gray dashed arrows are formal implications, black double arrows mark substantive equivalences, and transitive arrows are omitted.}
\label{fig:identifiability-landscape}
\end{figure}
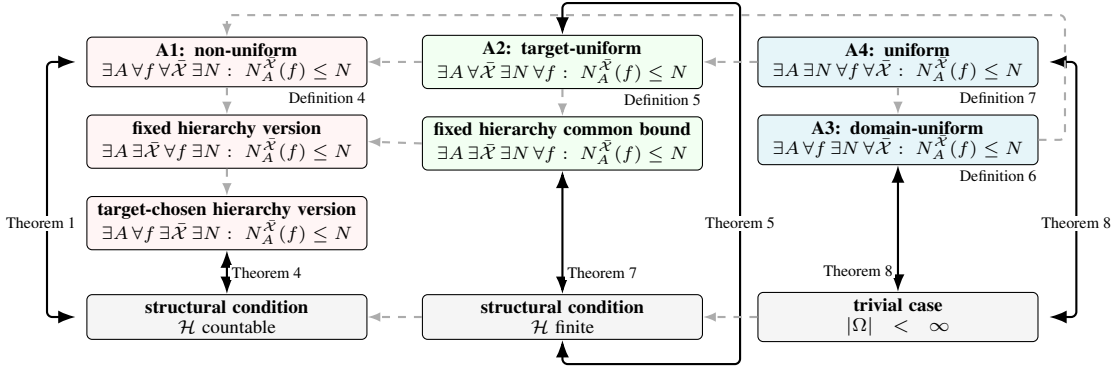

\subsection{Non-Uniform Domain-Dependent Generalization}
\label{sec:A1-characterization}

Very little, it turns out, is required for the weakest criterion. If we settle for \Aone, $\cH$ being countable is all we need.

\begin{theorem}[\Aone\ $\Leftrightarrow$ Countable $\cH$]
\label{thm:A1-countable}
$\cH$ is non-uniformly identifiable (\Aone, \Cref{def:A1}) if and only if $\cH$ is countable.
\end{theorem}

\begin{proof}[Proof sketch]
Enumerate $\cH$ and let $A$ be the consistent learner with this fixed tie-breaking order: after observing a prefix, it outputs the first hypothesis in the enumeration that is consistent with the observed labels. For $f=h_m$, each lower-index $h_j$ is eventually ruled out as $\bar{\cX}$ reveals its disagreement with $f$, leaving $h_m$ as the first such hypothesis. Conversely, fix any hierarchy. If a learner identifies every target after some finite prefix, then each target is associated with at least one finite labeled prefix on which the learner outputs it. Since there are only countably many finite labeled prefixes, $\cH$ must be countable.
\end{proof}

The countability characterization is still learner-dependent: it says that some chosen rule eventually outputs the target, not that the observed labels have ruled out every other hypothesis. A learner may identify the target by choosing one hypothesis consistent with the observed labels while other hypotheses still agree with those labels. The following example makes this separation explicit.

\begin{proposition}
\label{thm:a1-without-collapse}
There exist a countable hypothesis class $\cH$, a hierarchy $\bar{\cX}$, and a successful learner $A$ for $\cH$ over $\bar{\cX}$ (namely, for every $f \in \cH$, $N_A^{\bar{\cX}}(f)<\infty$) so that for some $f^* \in \cH$ for every finite level $n$, there exists some $g_n\in\cH$ so that $g_n|_{\cX_n}=f^*|_{\cX_n}$, and yet $g_n\neq f^*$ (namely, for some $x \in \Omega$, $f^*(x)\neq g_n(x)$).
\end{proposition}

\begin{proof}[Proof sketch]
Let $\Omega=\mathbb{N}$ and $\cH$ contain the all-zero function $0$ together with the singleton indicators $h_i=\mathds{1}_{\{i\}}$. Take the natural hierarchy $\cX_n=\{1,\ldots,n\}$ and the target $f=0$. A least-index learner that lists $0$ first identifies $0$ immediately, so $N_A^{\bar{\cX}}(0)<\infty$. But for every $n$ and every $j>n$, the singleton $h_j$ agrees with $0$ on $\cX_n$. Thus no finite prefix by itself rules out all alternatives to $0$.
\end{proof}

For the class in \Cref{thm:a1-without-collapse}, every finite set misses some singleton hypothesis, so the target $0$ is not finitely isolated. More generally, finite collapse of the version space is equivalent to the existence of a finite set of domain points that separates the target from all alternatives; this existence does not depend on the ordering of the hierarchy.

\begin{proposition}[Finite isolation is hierarchy-invariant]
\label{prop:finite-isolation-hierarchy-invariant}
Fix $f\in\cH$. For a hierarchy $\bar{\cX}$, write $V_{\cH,n}^{\bar{\cX}}(f):=\{g\in\cH:g|_{\cX_n}=f|_{\cX_n}\}$ for the version space around $f$ at level $n$. The following are equivalent:
\begin{enumerate}[label=(\roman*),itemsep=0.15ex,topsep=0.35ex]
    \item for some hierarchy $\bar{\cX}$, there is a finite level $n$ such that $V_{\cH,n}^{\bar{\cX}}(f)=\{f\}$;
    \item for every hierarchy $\bar{\cX}$, there is a finite level $n$ such that $V_{\cH,n}^{\bar{\cX}}(f)=\{f\}$;
    \item there is a finite set $S\subseteq\Omega$ such that every $g\in\cH\setminus\{f\}$ differs from $f$ on some point of $S$.
\end{enumerate}
\end{proposition}

One topological way to understand this is the following. For a finite $S\subseteq\Omega$ and $h\in\cH$, let $[h]_S:=\{g\in\cH:g|_S=h|_S\}$ be the set of hypotheses that agree with $h$ on the finite observation set $S$. These finite-observation cells generate a topology on $\cH$. For any finite-level hierarchy that exhausts $\Omega$, the same topology is generated by the level version spaces $V_{\cH,n}^{\bar{\cX}}(h)$. In this topology, saying that $\{f\}$ is open is exactly saying that some finite observation set isolates $f$, or equivalently that the version space around $f$ collapses at a finite level.

The next theorem gives several ways of relaxing the hierarchy quantifier in \Aone: one may fix a hierarchy in advance, or even choose the hierarchy separately for each target. These relaxations have the same threshold because they still ask only for target-by-target success. Thus the following statements are equivalent.

\begin{theorem}[Equivalent conditions for \Aone]
\label{thm:A1-equivalent-conditions}
The following are equivalent:
\begin{enumerate}[label=(\roman*),itemsep=0.15ex,topsep=0.35ex]
  \item $\exists A\;\forall f\in\cH\;\forall \bar{\cX}\;\exists N\in\mathbb{N}_{\geq 1}:\; N_A^{\bar{\cX}}(f)\leq N$ (that is, $\cH$ satisfies \Aone, \Cref{def:A1});
  \item $\exists A\;\exists \bar{\cX}\;\forall f\in\cH\;\exists N\in\mathbb{N}_{\geq 1}:\; N_A^{\bar{\cX}}(f)\leq N$;
  \item $\exists A\;\forall f\in\cH\;\exists \bar{\cX}\;\exists N\in\mathbb{N}_{\geq 1}:\; N_A^{\bar{\cX}}(f)\leq N$;\qquad \textup{(iv)} $\cH$ is countable.
\end{enumerate}
\end{theorem}

\subsection{Target-Uniform Domain-Dependent Generalization}
\label{sec:A2-characterization}

\Atwo\ asks for a different kind of success. Once the hierarchy is fixed, the same prefix level must work for every target: the amount of evidence is independent of which $f\in\cH$ generated the labels.

\begin{theorem}[\Atwo\ $\Leftrightarrow$ Finite $\cH$]
\label{thm:A2-finite}
$\cH$ is target-uniformly identifiable (\Atwo, \Cref{def:A2}) if and only if $\cH$ is finite.
\end{theorem}

\begin{proof}[Proof sketch]
With $\cH$ finite, the least-index learner from \Cref{thm:A1-countable} achieves a common bound on any hierarchy: $\max_{f\in\cH} N_A^{\bar{\cX}}(f)<\infty$. Conversely, a common bound $N$ on a single hierarchy makes the labeled-prefix map $f\mapsto f|_{\cX_N}$ injective on $\cH$; otherwise two targets with the same prefix would give the learner the same input at level $N$. Hence $|\cH|\leq 2^{|\cX_N|}$.
\end{proof}

The proof isolates the mechanism behind \Atwo: at a common level, two targets with the same prefix are indistinguishable to any learner. The relevant quantity is therefore the first prefix that separates the class.

\begin{definition}[Class domain complexity]
\label{def:class-domain-complexity-main}
For a hypothesis class $\cH$ and a hierarchy $\bar{\cX}$, the class domain complexity is
\begin{equation}
  N^{\bar{\cX}}(\cH)
  :=
  \min\{n\in\mathbb{N}_{\geq 1}:\forall f,g\in\cH,\ f\neq g \Longrightarrow \exists x\in\cX_n: f(x)\neq g(x)\}.
\end{equation}
If no such finite level exists, set $N^{\bar{\cX}}(\cH):=\infty$.
\end{definition}

\begin{proposition}[Optimality of class domain complexity]
\label{prop:class-domain-complexity-optimal}
For every hierarchy $\bar{\cX}$, $N^{\bar{\cX}}(\cH)=\min_A \sup_{f\in\cH} N_A^{\bar{\cX}}(f)$, where $A$ ranges over learners and both sides take values in $\mathbb{N}_{\geq 1}\cup\{\infty\}$.
\end{proposition}

Thus the optimal fixed-hierarchy \Atwo\ bound is algorithm-agnostic, and the variants below are equivalent: a common bound along one hierarchy already forces $\cH$ to be finite, while finiteness gives a common bound along every hierarchy.

\begin{theorem}[Equivalent conditions for \Atwo]
\label{thm:target-uniform-equivalences}
\label{thm:fixed-hierarchy-equivalences}
\label{prop:fixed-hierarchy-common-bound}
The following are equivalent:
\begin{enumerate}[label=(\roman*),itemsep=0.15ex,topsep=0.35ex]
  \item $\exists A\;\forall \bar{\cX}\;\exists N\in\mathbb{N}_{\geq 1}\;\forall f\in\cH:\; N_A^{\bar{\cX}}(f)\leq N$ (that is, $\cH$ satisfies \Atwo, \Cref{def:A2});
  \item $\exists A\;\exists \bar{\cX}\;\exists N\in\mathbb{N}_{\geq 1}\;\forall f\in\cH:\; N_A^{\bar{\cX}}(f)\leq N$;
  \item $\exists \bar{\cX}:\; N^{\bar{\cX}}(\cH)<\infty$; \quad \textup{(iv)} $\forall \bar{\cX}:\; N^{\bar{\cX}}(\cH)<\infty$; \quad \textup{(v)} $\cH$ is finite.
\end{enumerate}
\end{theorem}

\subsection{Domain-Uniform Generalization}
\label{sec:A3-A4-characterization}

\begin{theorem}[\Athree/\Afour\ characterization]
\label{thm:A3-A4-characterization}
\label{thm:A3-brittle}
Assume $|\cH|>2$.
The following are equivalent:
\begin{enumerate}[label=(\roman*),itemsep=0.15ex,topsep=0.35ex]
  \item $\cH$ is uniformly identifiable (\Afour, \Cref{def:A4});
  \item $\cH$ is domain-uniformly identifiable (\Athree, \Cref{def:A3});
  \item $\Omega$ is finite.
\end{enumerate}
\end{theorem}

\begin{proof}[Proof sketch]
\emph{(i) implies (ii).} This is immediate. \emph{(iii) implies (i).} If $\Omega$ is finite, then every hierarchy reveals all of $\Omega$ by level $|\Omega|$. At that level the observed labels determine the target on the entire domain. A learner that outputs the unique hypothesis in $\cH$ consistent with the full-domain labels identifies every target under every hierarchy by level $|\Omega|$.

\emph{(ii) implies (iii).} Suppose \Athree\ holds through a learner $A$ and target-dependent levels $(N_f)_{f\in\cH}$. If two distinct hypotheses $f,g$ agreed on at least $M=\max\{N_f,N_g\}$ points, choose a finite set $S\subseteq\{f=g\}$ of size $M$ and a hierarchy with $\cX_M=S$. Since $f|_S=g|_S$, the learner receives the same input under both targets at level $M$, but \Athree\ requires the same output to equal both $f$ and $g$, a contradiction. Hence every pairwise agreement set is finite.

Since $|\cH|>2$, choose three distinct hypotheses $h_1,h_2,h_3\in\cH$. Write $\{h_i=h_j\}:=\{x\in\Omega:h_i(x)=h_j(x)\}$. At each $x\in\Omega$, two of the three binary labels must agree, so
\begin{equation}
  \{h_1 = h_2\} \,\cup\, \{h_1 = h_3\} \,\cup\, \{h_2 = h_3\} \,=\, \Omega.
\end{equation}
The three sets on the left are finite by the previous paragraph. Hence $\Omega$ is finite.
\end{proof}

Under the standing assumption of this paper, $\Omega$ is countably infinite. Thus for every class with $|\cH|>2$, both \Athree\ and \Afour\ are impossible. The only remaining cases are degenerate. Singleton classes are trivial. For a two-element class $\cH=\{h,g\}$, domain-uniform (\Athree), equivalently uniform (\Afour), identification is possible exactly when $h$ and $g$ agree on only finitely many points: then any sufficiently large observed region contains a point where they differ, and the label at that point identifies the target.

\paragraph{Why Domain Generalization Is Hard.}
The obstruction above is not a usual PAC-learning obstruction from VC dimension, class size, or agnostic noise. It comes from the train/test partition itself. The useful way to read the obstruction is adversarial:

Suppose $\Omega$ is infinite and take any three distinct hypotheses $h_1,h_2,h_3\in\cH$. Then some pair must agree on infinitely many points. Indeed, if both $\{h_1=h_2\}$ and $\{h_1=h_3\}$ are finite, then their union is finite, so on infinitely many points $h_2$ and $h_3$ both take the binary label opposite to $h_1$. Hence $\{h_2=h_3\}$ is infinite. Thus, from any class with $|\cH|>2$, an adversary can choose two distinct hypotheses $f,g$ with an infinite agreement set and place arbitrarily many of those agreement points at the beginning of the hierarchy. On any such prefix $S\subseteq\{f=g\}$, the learner receives the same labeled training set under the two possible targets. Whatever it outputs on that input, at least one of $f$ and $g$ is not identified.

%% file: sections/complexity_stratified.tex
\section{Stratified Domain Complexity}
\label{sec:complexity-stratified-domain-complexity}

The characterizations above are qualitative: finite classes admit target-uniform bounds along a fixed hierarchy, while infinite classes do not. To study infinite model families quantitatively, we stratify them by model complexity and ask how the best bound grows with the allowed complexity.

\paragraph{Complexity Sublevels.}
A natural refinement is to keep track of model complexity. Many representation classes of interest are not just infinite sets of functions: programs, automata, grammars, and implemented neural-network families come with finite descriptions or size parameters, such as description length,\footnote{For a prefix-code description length, Kraft's inequality implies that only finitely many descriptions have code length at most any fixed $s$. We use only this finite-slice consequence, not an MDL or probabilistic argument.} program length, number of states, grammar size, depth, width, or number of parameters. Such a scale $c:\cH\to\mathbb{N}$ organizes the class into bounded-complexity slices.

We now formalize this fixed-hierarchy quantity. Fix a hierarchy $\bar{\cX}$ on $\Omega$, a hypothesis class $\cH\subseteq\{0,1\}^{\Omega}$, and a model-complexity function $c:\cH\to\mathbb{N}$. For $s\in\mathbb{N}$, the \emph{sublevel class} of model complexity at most $s$ is
\begin{equation}
  \cH_{\leq s}^{c}:=\{f\in\cH:c(f)\leq s\}.
  \label{eq:complexity-sublevel-class}
\end{equation}
Throughout this section, we consider complexity scales whose sublevels are finite and exhaustive:
\begin{equation}
  |\cH_{\leq s}^{c}|<\infty\ \text{ for every }s\in\mathbb{N},
  \qquad
  \bigcup_{s\in\mathbb{N}}\cH_{\leq s}^{c}=\cH.
  \label{eq:finite-sublevel-exhaustion}
\end{equation}

\paragraph{Sublevel Bounds.}
Abstractly, every countable class admits such a stratification by enumeration. In applications, however, the scale $c$ is chosen because it reflects a representation of the hypotheses. The quantity of interest is the best target-uniform bound on each sublevel class.

\begin{definition}[Sublevel domain complexity]
\label{def:fixed-hierarchy-sublevel-domain-complexity}
The \emph{sublevel domain complexity} of $(\cH,c)$ along $\bar{\cX}$ is the function $N_{(\cH,c)}^{\bar{\cX}}:\mathbb{N}\to\mathbb{N}_{\geq 1}\cup\{\infty\}$ defined by
\begin{equation}
  \begin{aligned}
  N_{(\cH,c)}^{\bar{\cX}}(s)
  := N^{\bar{\cX}}(\cH_{\leq s}^{c})
  =
  \min\big\{i\in\mathbb{N}_{\geq 1}: \forall h\neq h'\in\cH_{\leq s}^{c},\ \exists x\in\cX_i
  \text{ s.t. } h(x)\neq h'(x)\big\}.
  \end{aligned}
  \label{eq:fixed-hierarchy-sublevel-domain-complexity}
\end{equation}
If no such $i$ exists, set $N_{(\cH,c)}^{\bar{\cX}}(s):=\infty$.
\end{definition}

Under the finite-sublevel assumption, each sublevel class has some finite target-uniform identifying level along the fixed hierarchy. The next result shows that the smallest possible such bound is exactly the sublevel domain complexity just defined.

\begin{theorem}[Sublevel domain complexity is the optimal rate]
\label{thm:sublevel-domain-complexity-optimal-rate}
Fix a hierarchy $\bar{\cX}$, a class $\cH\subseteq\{0,1\}^{\Omega}$, and a complexity scale $c:\cH\to\mathbb{N}$ satisfying \Cref{eq:finite-sublevel-exhaustion}. Then, for every $s\in\mathbb{N}$ with $\cH_{\leq s}^{c}\neq\emptyset$,
\begin{equation}
  N_{(\cH,c)}^{\bar{\cX}}(s)
  =
  \min_A \sup_{f\in\cH_{\leq s}^{c}} N_A^{\bar{\cX}}(f),
\end{equation}
where $A$ ranges over learners. If, in addition, $|\cH_{\leq s}^{c}|\to\infty$ as $s\to\infty$, then $N_{(\cH,c)}^{\bar{\cX}}(s)\to\infty$.
\end{theorem}

Thus finite sublevels give the qualitative guarantee, while \Cref{thm:sublevel-domain-complexity-optimal-rate} identifies the best possible target-uniform bound at each complexity level. When the sublevel classes grow, these bounds must also diverge. The remaining interesting question is how fast the sublevel domain complexity function grows.

%% file: sections/length_generalization.tex
\section{Length Generalization as a Case Study}
\label{sec:length-gen}

Length generalization is the fixed-hierarchy case in which the hierarchy is ordered by input length and the quantitative question is how sublevel domain complexity grows with model complexity.

Concretely, length generalization asks whether a model trained on short inputs succeeds on longer inputs. This question is especially visible for Transformers~\citep{vaswani2017attention} trained on algorithmic tasks, where models can fit all short instances yet fail sharply on longer ones~\citep{anil2022exploring,lee2023teaching,zhou2023algorithms}.

\begin{definition}[Length hierarchy]
\label{def:length-hierarchy}
Formally, fix a finite alphabet $\Sigma$ and consider the hierarchy
\begin{equation}
  \Omega:=\Sigma^+,
  \qquad
  \bar{\Sigma}:=(\Sigma^{\leq N})_{N\geq 0},
  \qquad
  \Sigma^{\leq 0}:=\emptyset,
  \qquad
  \Sigma^{\leq N}:=\Sigma\cup\Sigma^2\cup\cdots\cup\Sigma^N
  \quad (N\geq 1).
\end{equation}
A learner trained on lengths at most $N$ observes all labels on $\Sigma^{\leq N}$ and is evaluated on all of $\Sigma^+$. Thus length generalization is exact identification under the fixed cumulative hierarchy $\bar{\Sigma}$.
\end{definition}

We write the representation class as $\cH\subseteq\{0,1\}^{\Sigma^+}$.\footnote{The binary-valued language notation is only for convenience; the same setup extends directly to finite-output maps such as functions $\Sigma^*\to\Sigma$.} Here $\cH$ may be a class of languages, programs, or Transformer-realizable predictors, and $c$ may encode description length, program size, Transformer depth, number of heads, parameter precision, norm bounds, or related architectural parameters. The quantitative content of non-asymptotic length generalization lies in estimating the growth of $N_{(\cH,c)}^{\bar{\Sigma}}(s)$ for concrete choices of $(\cH,c)$.

\paragraph{Existing Results.}
\citet{huang2025formal} give a qualitative fixed-hierarchy, target-wise guarantee for Transformer length generalization: under an idealized regularized inference procedure, targets in their Limit-Transformer class are eventually identified along $\bar{\Sigma}$. This does not estimate a sublevel domain-complexity growth rate.

More quantitatively, \citet{chen2025nonasymptotic} estimate this sublevel domain complexity for several representation classes. For DFAs, the model class is regular languages and the complexity measure is the number of states; their bound is $2q-2$ for targets recognized by DFAs with at most $q$ states. For linear CFGs, the complexity measure is grammar description size, and the analogous sublevel bound has no computable upper bound in that size. For restricted C-RASP,\footnote{C-RASP programs compile to softmax Transformers~\citep{yang2024counting}; hence C-RASP gives a programmable lower bound on softmax-Transformer expressivity and is often used as a proxy for studying Transformer expressiveness.} the relevant parameters are precision $T$ in the one-layer Transformer setting, and precision $T$ together with $K$ heads in the two-layer setting; the resulting bounds are $O(T^2)$ and $O(T^{O(K)})$, respectively.

Later work continues the same calculation for more Transformer-facing classes. \citet{yang2026length} study the same sublevel domain-complexity function for C-RASP, proving uncomputability for full C-RASP and tight exponential growth for positive C-RASP. \citet{izzo2025quantitative} move to an approximate analogue: exact identification is replaced by approximate agreement on all longer inputs, with bounds in terms of norms, locality, periodicity, precision, and margins.

\paragraph{Beyond Length.}
The preceding results all concern the fixed length hierarchy $\bar{\Sigma}$. The uncomputability statements in this literature should not be interpreted as a peculiarity of ordering strings by length. The obstruction can already come from the representation of $\cH$: if semantic equivalence of descriptions is undecidable, then the sublevel domain-complexity function has no computable upper bound in the complexity parameter $s$.

To discuss computability, we use a represented class: a finite description $p\in\Pi$ denotes a computable hypothesis $h_p:\Omega\to\{0,1\}$, and $\kappa:\Pi\to\mathbb{N}$ is a computable complexity scale on descriptions. For a complexity bound $s$ and hierarchy $\bar{\cX}$, write
\begin{equation}
  \cH_{\leq s}^{\Pi,\kappa}:=\{h_p:p\in\Pi,\ \kappa(p)\leq s\},
  \qquad
  N_{(\Pi,\kappa)}^{\bar{\cX}}(s):=N^{\bar{\cX}}(\cH_{\leq s}^{\Pi,\kappa}).
\end{equation}
The quantity on the right is still the information-theoretic domain complexity from \Cref{def:fixed-hierarchy-sublevel-domain-complexity}. A hierarchy is computable if its finite levels can be listed, and two descriptions are semantically equivalent if their denoted hypotheses agree on all of $\Omega$.

\begin{proposition}[Computability obstruction]
\label{prop:computability-obstruction}
If semantic equivalence is undecidable for descriptions in $\Pi$, then for every computable exhaustive hierarchy $\bar{\cX}$ on $\Omega$, there is no computable function $B:\mathbb{N}\to\mathbb{N}$ such that
\begin{equation}
  N_{(\Pi,\kappa)}^{\bar{\cX}}(s)\leq B(s)
  \qquad\text{for all }s\in\mathbb{N},
\end{equation}
\end{proposition}

For classes such as linear CFGs, this recovers the absence of a computable non-asymptotic bound along $\bar{\Sigma}$ shown by \citet{chen2025nonasymptotic}, and holds for any computable exhaustive hierarchy. The computability requirement on the hierarchy is essential: without it, the hierarchy itself can encode the finite witnesses that separate the relevant hypotheses.

\begin{proposition}[Necessity of computability]
\label{prop:noncomputable-hierarchy-sublevel-bound}
Let $\Omega$ be countably infinite, and let $\cH=\{h_1,h_2,\ldots\}\subseteq\{0,1\}^{\Omega}$ be countable. There exists an exhaustive hierarchy $\bar{\cX}$ such that $N^{\bar{\cX}}(\{h_1,\ldots,h_m\})\leq m$ for all $m\geq1$.
\end{proposition}

Thus, for a complexity scale $c$ with finite sublevels, if the enumeration is ordered so that $\cH_{\leq s}^{c}\subseteq\{h_1,\ldots,h_{b(s)}\}$ for some computable $b$, the same hierarchy gives $N_{(\cH,c)}^{\bar{\cX}}(s)\leq b(s)$ for all $s$. Hence the computability assumption in \Cref{prop:computability-obstruction} is essential.

Together, these observations separate two sources of difficulty. The computability obstruction above shows that non-computable rates can arise from the effective description of $\cH$ itself, provided the hierarchy is computable. The next section concerns the other source: even with $(\cH,c)$ fixed, the growth of the sublevel domain complexity need not be determined by this pair alone.

%% file: sections/no_free_lunch.tex
\section{No Free Lunch for Domain Generalization}
\label{sec:no-free-lunch-domain-generalization}

The preceding sections give meaningful quantitative notions once a hierarchy is fixed. For a target, a learner has an identification level $N_A^{\bar{\cX}}(f)$; for a finite class, $N^{\bar{\cX}}(\cH)$ is the optimal target-uniform bound; and for a complexity scale $c$, $N_{(\cH,c)}^{\bar{\cX}}(s)$ is the corresponding optimal bound on the sublevel class $\cH_{\leq s}^{c}$. All three are defined relative to the hierarchy $\bar{\cX}$.

This dependence is essential. \Cref{thm:A3-A4-characterization} already shows that, for classes with $|\cH|>2$, domain generalization uniformly over all hierarchies is possible only in the finite-domain case. We now formalize the same dependence as a no-free-lunch result: over an infinite domain, changing the hierarchy alone can make each of the fixed-hierarchy quantities above arbitrarily large.

\paragraph{Class-Level Bounds.}
We first ask this question before adding a complexity scale.

\begin{theorem}[No free lunch for domain generalization]
\label{thm:domain-generalization-nfl}
Let $\Omega$ be countably infinite. For any $\cH\subseteq\{0,1\}^{\Omega}$ with $|\cH|>2$ and every $M\in\mathbb{N}_{\geq 1}$, there exists a hierarchy $\bar{\cX}$, even one revealing a single new point at each level, such that, for every learner $A$, there exists a target $h\in\cH$ with
\begin{equation}
  N_A^{\bar{\cX}}(h)>M.
\end{equation}
When $\cH$ is finite, this is equivalent to the class-domain-complexity statement that, for every $M\in\mathbb{N}_{\geq 1}$, there exists a hierarchy $\bar{\cX}$ such that $N^{\bar{\cX}}(\cH)>M$.
\end{theorem}

Equivalently, the theorem can be read as a train/test no-free-lunch statement. Fix any hypothesis class $\cH$ with $|\cH|>2$, no matter how small, and fix any desired training size $M$, no matter how large. There exists a partition $\Omega=\cX_{\mathrm{tr}}\sqcup\cX_{\mathrm{te}}$ with $|\cX_{\mathrm{tr}}|=M$ such that no learner succeeds uniformly for every target. This does not contradict the fixed-hierarchy characterizations: countability gives target-wise existence, and finiteness gives target-uniform existence, after $\bar{\cX}$ is fixed. The theorem says that changing $\bar{\cX}$ can push the required level past any prescribed $M$.

\paragraph{Complexity-Stratified Bounds.}
The same question can be asked after stratifying the class by complexity. Once $\bar{\cX}$ is fixed, finite sublevel classes give well-defined optimal target-uniform bounds $N_{(\cH,c)}^{\bar{\cX}}(s)$. Can the growth of this function be controlled by $(\cH,c)$ alone?

\begin{corollary}[Complexity-stratified no free lunch]
\label{cor:complexity-stratified-nfl}
Let $\Omega$ be countably infinite, let $\cH\subseteq\{0,1\}^{\Omega}$ be countably infinite, and let $c:\cH\to\mathbb{N}$ have finite sublevel classes. For every nondecreasing $B:\mathbb{N}\to\mathbb{N}_{\geq 1}$ with $B(s)\to\infty$, there exist a hierarchy $\bar{\cX}$ on $\Omega$ and $s_0\in\mathbb{N}$ such that, for every $s\geq s_0$,
\begin{equation}
  N_{(\cH,c)}^{\bar{\cX}}(s)>B(s).
\end{equation}
\end{corollary}

The key point is quantitative: with $(\cH,c)$ fixed, changing the domain hierarchy can make the sublevel domain-complexity function grow faster than any prescribed unbounded rate $B$. In particular, for some hierarchy, this function has no computable upper bound.

For length generalization, the hierarchy is the length hierarchy $\bar{\Sigma}$ from \Cref{def:length-hierarchy}. A bound for a Transformer-realizable class under this hierarchy is therefore a statement about training on all strings up to a given length and testing on longer strings. It does not automatically imply the corresponding bound for the same Transformer architecture and the same precision scale under another hierarchy on $\Sigma^+$, for example one ordered by reasoning-trajectory length, recursion depth, theorem difficulty, or planning horizon. The corollary says that such a replacement can make the bound arbitrarily large, even noncomputably so.

\paragraph{Comparison with Classical NFL.}
Classical learning-theoretic no-free-lunch theorems are usually read as warnings about hypothesis classes~\citep{shalev2014understanding}. If the learner is asked to compete with all possible labelings, there is no structure to exploit and no uniform generalization theorem to be had. Our result moves the warning to a second object that is often kept offstage: the structure connecting what is in domain to what is out of domain. Fixing the hypothesis class is therefore not enough. A domain-generalization statement also needs to say what makes the observed part informative about the unobserved part. Counting hypotheses alone is therefore not enough.

Therefore, domain-generalization guarantees cannot be characterized by the concept class alone. They must be properties of the pair $(\cH,\bar{\cX})$: the hypotheses together with the hierarchy that determines how observed regions are related to unobserved ones.

%% file: sections/extensions.tex
\section{Extended Discussions}
\label{sec:extended-discussions}

We consider four extensions: online mistake bounds, approximate prediction, distributional support mismatch, and positive-only presentations. 

\subsection{Online Learning}
\label{sec:online-learning}

Online learning keeps the same hierarchy, but replaces the identifying level by a mistake count. Before level $n$ is revealed, the prediction rule has observed the labels on $\cX_{n-1}$ and must predict the labels on the new increment $\Delta_n=\cX_n\setminus\cX_{n-1}$. We count one mistake at level $n$ if any of these predictions is wrong.

This is weaker than identification. It counts errors only when they are exposed on a newly revealed increment; identification asks whether the labels already observed determine the target on all of $\Omega$.

As in \Cref{def:main-learner}, the online prediction rule may output an arbitrary binary predictor in $\{0,1\}^{\Omega}$, not necessarily an element of $\cH$.

\begin{definition}[Uniform mistake bound]
\label{def:uniform-mistake-bound}
For an online prediction rule $A$, target $f \in \cH$, and hierarchy $\bar{\cX}$, define the \emph{mistake count}
\begin{equation}
  M_A(f, \bar{\cX})
  :=
  \sum_{n \geq 1}
  \mathbf{1}\!\left[
    A(f|_{\cX_{n-1}})|_{\cX_n\setminus\cX_{n-1}}
    \not\equiv
    f|_{\cX_n\setminus\cX_{n-1}}
  \right],
\end{equation}
where $\cX_0=\emptyset$. We say $\cH$ admits a \emph{uniform mistake bound} if there exist an online prediction rule $A$ and a finite $M \in \mathbb{N}$ such that $M_A(f,\bar{\cX}) \leq M$ for every $f\in\cH$ and every hierarchy $\bar{\cX}$.
\end{definition}

Let $\Ldim(\cH)$ denote the Littlestone dimension of $\cH$, namely the supremum of the depths of complete binary trees shattered by $\cH$. The classical mistake-bound theorem of \citet{littlestone1988learning} applies to the criterion above.

\begin{theorem}[Littlestone's mistake-bound theorem]
\label{thm:Ldim-mistake-bound}
$\cH$ admits a uniform mistake bound if and only if $\Ldim(\cH)<\infty$. Moreover, the optimal mistake bound is exactly $\Ldim(\cH)$.
\end{theorem}

Thus Littlestone dimension controls the number of mistaken revealed increments. It does not control the level at which the target has been identified.

\paragraph{From Mistakes to Identification.}
The mistake-bound theorem is uniform over both targets and hierarchies, just as \Afour\ is. The difference is what the bound promises: \Afour\ gives a level after which the target has been identified, while a mistake bound only limits how many revealed increments can be predicted incorrectly. The next proposition places this online criterion between the finite and countable cases characterized in \Cref{sec:identification}.

\begin{proposition}[Intermediate mistake complexity]
\label{prop:mistake-intermediate}
Assume $\Omega$ is countable. Then
\begin{equation}
  |\cH|<\infty
  \quad\Longrightarrow\quad
  \Ldim(\cH)<\infty
  \quad\Longrightarrow\quad
  \cH \text{ is countable}.
\end{equation}
\end{proposition}

Thus, over a countable domain, finite Littlestone dimension sits between finiteness and countability. The same counting argument extends to any infinite domain: if $\Ldim(\cH)=d<\infty$, then $|\cH|\leq|\Omega|$. Fix a well-order of $\Omega$ and run a bounded-mistake learner; each target is determined by its at-most-$d$ mistake locations. Since an infinite $\Omega$ has only $|\Omega|$ subsets of size at most $d$, the bound follows.

Combining \Cref{prop:mistake-intermediate} with \Cref{thm:A2-finite,thm:Ldim-mistake-bound,thm:A1-countable} gives
\begin{equation}
  \Atwo\quad\Longrightarrow\quad \text{uniform mistake bound}\quad\Longrightarrow\quad \Aone.
\end{equation}
Thus online learning is weaker than target-uniform identification, even when the same mistake bound must work for every hierarchy. A mistake bound controls how many times disagreement can be revealed as error, but not how long the hierarchy can postpone the first point that distinguishes the remaining possible targets.

\begin{example}[One mistake does not imply target-uniform identification]
\label{ex:singleton-ldim-one}
Let $\Omega=\{x_1,x_2,\ldots\}$ and consider the class of singleton indicators
\begin{equation}
  \cH_{\mathrm{sing}}
  :=
  \{h_i : i\geq 1\},
  \qquad
  h_i(x_j) = \mathbf{1}[i=j].
\end{equation}
Then $\Ldim(\cH_{\mathrm{sing}})=1$, but $\cH_{\mathrm{sing}}$ does not satisfy \Atwo.
\end{example}

For this class, online prediction is easy: predict zero until the unique positive point appears. At most one mistake is possible. Identification is different. Along the natural hierarchy $\cX_N=\{x_1,\ldots,x_N\}$, all targets $h_i$ with $i>N$ have the same all-zero labels at level $N$. Thus no common finite level can force the learner to know which singleton target is correct.

\subsection{Approximation}
\label{sec:approximation}

A second comparison keeps the hierarchy but weakens exact identification to small $P$-error. So far, success meant exact identification: after some level, the learner must recover the target on all of $\Omega$. Here we only ask for small error under a fixed evaluation distribution. The question is whether the hierarchy-uniform impossibility from exact identification persists under this weaker requirement.

Fix a distribution $P$ on $\Omega$. For a target $f\in\cH$ and predictor $h\in\{0,1\}^{\Omega}$, the $P$-error of $h$ relative to $f$ is
\begin{equation}
  L_{P,f}(h):=P(\{x\in\Omega:h(x)\neq f(x)\}).
\end{equation}

\begin{definition}[$\varepsilon$-approximate domain complexity]
\label{def:eps-approx-domain-complexity}
For $\varepsilon>0$, learner $A$, hierarchy $\bar{\cX}$, evaluation distribution $P$, and target $f\in\cH$, define
\begin{equation}
  N_{A,\varepsilon}^{\bar{\cX},P}(f)
  :=
  \min\{i\in\mathbb{N}_{\geq 1}:\forall n\geq i,\; L_{P,f}(A(f|_{\cX_n}))\leq\varepsilon\},
\end{equation}
with value $\infty$ if no such finite $i$ exists. This is the first level after which every later predictor has $P$-error at most $\varepsilon$.
\end{definition}

For a fixed hierarchy, approximation has a simple tail-mass guarantee: every consistent learner satisfies
\begin{equation}
  L_{P,f}(A(f|_{\cX_n}))\leq P(\Omega\setminus\cX_n).
\end{equation}
Because the hierarchy is exhaustive, $P(\Omega\setminus\cX_n)\to0$. Thus fixed-hierarchy approximation does not require any combinatorial structure of $\cH$. The nontrivial question is whether one can choose a level bound before the hierarchy is known. Say that $P$ has full support if $P(x)>0$ for every $x\in\Omega$. Then every disagreement between distinct hypotheses has positive $P$-mass: if $u\neq v$, then $P(\{u\neq v\})>0$.

\begin{proposition}[Approximate no free lunch over hierarchies]
\label{prop:approximate-hierarchy-nfl}
Let $\Omega$ be countably infinite, let $P$ have full support on $\Omega$, and let $\cH\subseteq\{0,1\}^{\Omega}$ have $|\cH|>2$. Then there exists $\varepsilon>0$ such that, for every $M\in\mathbb{N}_{\geq 1}$, there is a hierarchy $\bar{\cX}$ for which every learner $A$ has some target $f\in\cH$ satisfying
\begin{equation}
  N_{A,\varepsilon}^{\bar{\cX},P}(f)>M.
\end{equation}
\end{proposition}

Thus allowing error does not produce a hierarchy-independent rate. For a fixed hierarchy and fixed evaluation distribution, consistency gives convergence in $P$-error. But before the hierarchy is fixed, the points on which hypotheses differ can be postponed for arbitrarily many levels, even when those points have positive $P$-mass.

\subsection{Out-of-Support Generalization}
\label{sec:out-of-support}

A third variant returns to a distributional observation model. The learner is no longer given a hierarchy of observed regions; it receives i.i.d.\ labeled samples from a training distribution $Q$, returns a binary predictor, and is evaluated under a possibly different distribution $P$. The hierarchy-uniform impossibility results above do not apply literally to this model. The same issue reappears as a support question: does training ever reveal the distinctions that evaluation can test?

The support of $Q$ is the region from which labels can ever be observed. Sampling reveals points in $\operatorname{supp}(Q)$ randomly, with repetitions, but never reveals labels outside that support. We measure error by the same $P$-error as in \Cref{sec:approximation}.

\paragraph{Full-Support Generalization.}
If $Q$ has full support, meaning $Q(x)>0$ for every $x\in\Omega$, no point is permanently outside training. Every point that can affect the $P$-error is at least possible to observe. For finite classes, this gives a target-uniform sample size.

\begin{theorem}[Full-support generalization]
\label{thm:full-support-source-main}
Let $\Omega$ be countably infinite, let $Q$ have full support on $\Omega$, fix any evaluation distribution $P$ on $\Omega$, and let $\cH\subseteq\{0,1\}^{\Omega}$ be finite. Then there exists a learner $A$ such that, for every $\varepsilon,\delta>0$, there is a finite $N$ such that, for every target $f\in\cH$ and every $n\geq N$,
\begin{equation}
  \mathbb{P}_{(x_1,\ldots,x_n)\sim Q^n}
  \!\left[
    L_{P,f}\!\left(A\!\left(\big(x_i,f(x_i)\big)_{i=1}^n\right)\right)\leq\varepsilon
  \right]
  \geq
  1-\delta.
  \label{eq:full-support-source-main-bound}
\end{equation}
\end{theorem}

Full-support sampling removes the permanent hiding place for any fixed disagreement set. Since a finite class has only finitely many pairwise disagreement sets, one sample size can make all of them appear with high probability.

\paragraph{Out-of-Support.}
If $Q$ is not full support, some labels are never observed. The next theorem says that this gap can break the same target-uniform conclusion even when $Q$ has infinite support: the learner may see infinitely many possible training points, but still not the part of $\Omega$ on which evaluation distinguishes two targets.

\begin{theorem}[Out-of-support generalization can fail]
\label{thm:infinite-support-source-failure}
Let $\Omega$ be countably infinite, let $P$ have full support on $\Omega$, and let $\cH\subseteq\{0,1\}^{\Omega}$ with $|\cH|\geq3$. Then there exist a distribution $Q$ with infinite support and an $\varepsilon>0$ such that, for every learner $A$, every $N\in\mathbb{N}_{\geq 1}$, and every $n\geq N$, there exists a target $f\in\cH$ satisfying
\begin{equation}
  \mathbb{P}_{(x_1,\ldots,x_n)\sim Q^n}
  \!\left[
    L_{P,f}\!\left(A\!\left(\big(x_i,f(x_i)\big)_{i=1}^n\right)\right)>\varepsilon
  \right]
  \geq \frac{1}{2}.
  \label{eq:infinite-support-main-bound}
\end{equation}
\end{theorem}

Thus the relevant condition is not how many points $Q$ can sample from. It is whether the training support includes the distinctions that $P$-error will test. Full support guarantees this; an infinite support that is not full need not.

\subsection{Positive-Only Presentations}
\label{sec:positive-only-presentations}

Finally, consider a nearby observation model, familiar from identification from text and generation in the limit~\citep{gold1967language,kleinberg2024languagegeneration}, in which the learner sees only positive examples. For $f\in\cH$, write
\begin{equation}
  \mathrm{Pos}(f):=\{x\in\Omega:f(x)=1\}
\end{equation}
for its positive region. A positive presentation of $f$ is a strictly increasing sequence $\emptyset=S_0\subsetneq S_1\subsetneq S_2\subsetneq\cdots$ of finite subsets of $\mathrm{Pos}(f)$ whose union is $\mathrm{Pos}(f)$. We restrict to the nontrivial case in which every $\mathrm{Pos}(f)$ is infinite.
Write $\mathrm{Pres}^{+}(f)$ for the set of all such presentations.

We ask for the same fully uniform exact-identification guarantee as before, but with positive sets replacing labeled finite domains. Here $A$ denotes a positive-only learner whose input is a finite positive set:
\begin{equation}
  \exists A\;\exists N<\infty\;\forall f\in\cH\;\forall \bar S\in\mathrm{Pres}^{+}(f):\;
  \forall t\geq N,\; A(S_t)=f.
  \label{eq:positive-only-uniform-id}
\end{equation}
The relevant condition is finite pairwise positive overlap: distinct targets cannot share arbitrarily many positive points.

\begin{proposition}[Uniform positive-only identification]
\label{prop:positive-only-uniform-identification}
Assume $\mathrm{Pos}(f)$ is infinite for every $f\in\cH$. Then the fully uniform positive-only criterion in \Cref{eq:positive-only-uniform-id} holds if and only if
\begin{equation}
  \sup_{\substack{f,g\in\cH\\ f\neq g}}|\mathrm{Pos}(f)\cap \mathrm{Pos}(g)|<\infty,
\end{equation}
with the convention that the supremum over the empty set is $0$. Moreover, if this supremum is $d$, then $N=d+1$ suffices.
\end{proposition}

The condition is not finiteness. For example, take $\Omega=\mathbb{N}\times\mathbb{N}$ and, for $a,b\in\mathbb{N}$, let
\begin{equation}
  G_{a,b}:=\{(t,at+b):t\in\mathbb{N}\},
\end{equation}
and let $f_{a,b}$ be the indicator of $G_{a,b}$. The class $\{f_{a,b}:a,b\in\mathbb{N}\}$ is countably infinite, and every positive region is infinite. Yet two distinct graphs intersect in at most one point, so two positive observations already certify the target against every competitor.

In our terms, this restricts the presentation to positive examples while keeping exact identification and full uniformity. Under this restriction, the strongest uniform statement can hold even for infinite classes.

This positive-only model is closely related to generation in the limit~\citep{kleinberg2024languagegeneration}. There too, the learner sees a growing finite subset of the target language, but only needs to generate one new valid positive example. Since the target is unknown, that point must be positive for every consistent hypothesis; closure dimension characterizes when this can be done uniformly~\citep{raman2025generation}. The proposition above shows that the stronger goal of recovering the whole language, equivalently generating all positives, can also hold uniformly. 

%% file: appendix/app_a_notation.tex
\section{Auxiliary Definitions}
\label{app:notation}

The appendix uses the following derived notation.

For two hypotheses $f,g\in\{0,1\}^{\Omega}$, write
\begin{equation}
  \Eq(f,g):=\{x\in\Omega:f(x)=g(x)\},
\end{equation}
also denoted $\{f=g\}$ when this is typographically lighter.

\begin{definition}[Version space and pointwise class complexity]
Fix a hypothesis class $\cH$, a hierarchy $\bar{\cX}$, and a target $f\in\cH$. The version space at level $n$ is
\begin{equation}
  V_{\cH,n}^{\bar{\cX}}(f)
  :=
  \{g\in\cH:g|_{\cX_n}=f|_{\cX_n}\}.
\end{equation}
The pointwise algorithm-agnostic domain complexity is
\begin{align}
  N_{\cH}^{\bar{\cX}}(f)
  &:=
  \min\{n\in\mathbb{N}_{\geq 1}:V_{\cH,n}^{\bar{\cX}}(f)=\{f\}\} \\
  &=
  \min\{n\in\mathbb{N}_{\geq 1}:\forall g\in\cH\setminus\{f\},\ \exists x\in\cX_n:g(x)\neq f(x)\}.
\end{align}
If no such finite level exists, set $N_{\cH}^{\bar{\cX}}(f):=\infty$.
\end{definition}

%% file: appendix/app_b_characterization_proofs.tex
\section{Proofs for Exact Characterization}
\label{app:characterization-proofs}

This appendix proves the characterization results from \Cref{sec:characterization}. The proofs repeatedly use the same elementary observation: if two targets induce the same labeled sample, then any deterministic learner receives the same input under the two targets and must return the same hypothesis.

\subsection{Non-Uniform Domain-Dependent Generalization}
\label{app:A1-proofs}

\begin{proof}[\textbf{Proof of \Cref{thm:A1-countable}}]
If $\cH=\emptyset$, the conclusion is vacuous. Assume henceforth that $\cH$ is nonempty.
Assume first that $\cH$ is countable. Fix an enumeration
\begin{equation}
  \cH=\{h_1,h_2,\ldots\}.
\end{equation}
Define the learner $A$ as follows. On any labeled finite sample, if there is at least one hypothesis in $\cH$ consistent with that sample, output the least-index consistent hypothesis; if there is none, output $h_1$. Only the consistent case matters on samples generated by targets in $\cH$.

Fix a target $f=h_m$ and a hierarchy $\bar{\cX}$. For every $j<m$, the hypotheses $h_j$ and $f$ are distinct, so choose a point $x_j\in\Omega$ such that $h_j(x_j)\neq f(x_j)$. The finite set $\{x_j:j<m\}$ is contained in some prefix $\cX_N$ because $\bar{\cX}$ exhausts $\Omega$. For every $n\geq N$, each lower-index hypothesis $h_j$ with $j<m$ disagrees with $f$ on $\cX_n$, while $f$ itself remains consistent with $f|_{\cX_n}$. Hence the least-index consistent hypothesis is exactly $f$, so $A(f|_{\cX_n})=f$ for all $n\geq N$. Thus $N_A^{\bar{\cX}}(f)\leq N$, proving \Aone.

Conversely, suppose $\cH$ satisfies \Aone, witnessed by a learner $A$. Fix a point-by-point hierarchy $\bar{\cX}$, say $\cX_n=\{x_1,\ldots,x_n\}$. For each $f\in\cH$, choose a finite level $N_f$ such that $A(f|_{\cX_{N_f}})=f$, which is possible because $N_A^{\bar{\cX}}(f)<\infty$. Consider the map
\begin{equation}
  f\longmapsto \bigl(N_f, f|_{\cX_{N_f}}\bigr).
\end{equation}
This map is injective. Indeed, if two targets $f,g$ have the same image, then they give $A$ the same labeled sample at the same level; hence
\begin{equation}
  f=A(f|_{\cX_{N_f}})=A(g|_{\cX_{N_g}})=g.
\end{equation}
For each fixed $n$, there are only finitely many binary labelings of $\cX_n$, and there are countably many choices of $n$. Therefore the range of this injective map is countable, and so $\cH$ is countable.
\end{proof}

\begin{proof}[\textbf{Proof of \Cref{thm:a1-without-collapse}}]
Let $\Omega=\mathbb{N}$ and
\begin{equation}
  \cH=\{0\}\cup\{h_i:i\geq 1\},
  \qquad
  h_i(j)=\mathbf{1}[i=j],
\end{equation}
where $0$ is the all-zero function. Let the hierarchy be $\cX_n=\{1,\ldots,n\}$. Use the least-index consistent learner with $0$ listed first.

For the target $f=0$, every observed prefix is all zero, so $0$ is consistent at every level and is listed first. Hence $A(0|_{\cX_n})=0$ for every $n\geq 1$, and $N_A^{\bar{\cX}}(0)=1$.

However, for every finite $n$ and every $j>n$, the singleton indicator $h_j$ agrees with $0$ on $\cX_n$. Thus the version space around $0$ never collapses at any finite level, even though the learner has already identified the target.
\end{proof}

\begin{proof}[\textbf{Proof of \Cref{prop:finite-isolation-hierarchy-invariant}}]
The implication (ii)$\Rightarrow$(i) is immediate. If (i) holds, then any level $\cX_n$ with $V_{\cH,n}^{\bar{\cX}}(f)=\{f\}$ is such a finite set $S$. If (iii) holds, then every hierarchy eventually contains $S$, and at that level the version space around $f$ is $\{f\}$.

For the topological reformulation, the finite-observation cells $[h]_S$ form a basis: finite intersections of such cells are either empty or another finite-observation cell. Since every finite-level hierarchy that exhausts $\Omega$ eventually contains any fixed finite $S\subseteq\Omega$, its level version spaces generate exactly the same topology. Finally, $\{f\}$ is open in this topology if and only if some finite-observation cell $[f]_S$ is equal to $\{f\}$, which is exactly finite isolation.
\end{proof}

\begin{proof}[\textbf{Proof of \Cref{thm:A1-equivalent-conditions}}]
The implication (i)$\Rightarrow$(ii) follows by fixing any hierarchy, and (ii)$\Rightarrow$(iii) follows by allowing the hierarchy to depend on the target.

Assume (iii). Then there is a learner $A$ such that for every $f\in\cH$ there are a hierarchy $\bar{\cX}^{\,f}$ and a finite level $N_f$ with
\begin{equation}
  A(f|_{\cX^{\,f}_{N_f}})=f.
\end{equation}
Assign to $f$ one such finite labeled sample $f|_{\cX^{\,f}_{N_f}}$. This assignment is injective: if two targets are assigned the same labeled finite sample, then $A$ receives the same input and must output both targets. Since $\Omega$ is countable, the collection of finite subsets of $\Omega$ is countable, and each finite subset has finitely many binary labelings. Hence there are only countably many finite labeled samples, so $\cH$ is countable.

Finally, if $\cH$ is countable, then (i) holds by \Cref{thm:A1-countable}. Therefore all the listed conditions are equivalent.
\end{proof}

\subsection{Target-Uniform Domain-Dependent Generalization}
\label{app:A2-proofs}

\begin{proof}[\textbf{Proof of \Cref{thm:A2-finite}}]
If $\cH=\emptyset$, the conclusion is vacuous. Assume henceforth that $\cH$ is nonempty.
Assume first that $\cH$ is finite. Use the least-index consistent learner from the proof of \Cref{thm:A1-countable}, for any fixed enumeration of the finite class. Fix a hierarchy $\bar{\cX}$. By the argument in \Cref{thm:A1-countable}, every target $f\in\cH$ has a finite identifying level $N_A^{\bar{\cX}}(f)$. Since $\cH$ is finite,
\begin{equation}
  N:=\max_{f\in\cH} N_A^{\bar{\cX}}(f)
\end{equation}
is finite and works for all targets along this hierarchy. The same learner is used for every hierarchy, so \Atwo\ holds.

Conversely, suppose \Atwo\ holds. Then there are a learner $A$ and, for any fixed hierarchy $\bar{\cX}$, a finite level $N$ such that $N_A^{\bar{\cX}}(f)\leq N$ for every $f\in\cH$. If two distinct targets $f,g\in\cH$ agreed on $\cX_N$, then the learner would receive the same input at level $N$ under both targets. Since $N_A^{\bar{\cX}}(f)\leq N$ and $N_A^{\bar{\cX}}(g)\leq N$, it would have to output both $f$ and $g$ on that same input, a contradiction. Thus the restriction map
\begin{equation}
  f\longmapsto f|_{\cX_N}
\end{equation}
is injective on $\cH$. Because $\cX_N$ is finite, it has only finitely many binary labelings. Hence $\cH$ is finite.
\end{proof}

\begin{proof}[\textbf{Proof of \Cref{prop:class-domain-complexity-optimal}}]
Fix a hierarchy $\bar{\cX}$.
If $\cH=\emptyset$, the identity is vacuous. Assume henceforth that $\cH$ is nonempty.

First suppose a learner $A$ identifies every target in $\cH$ by some finite level $M$, meaning
\begin{equation}
  \sup_{f\in\cH}N_A^{\bar{\cX}}(f)\leq M.
\end{equation}
If two distinct hypotheses $f,g\in\cH$ agreed on $\cX_M$, then $A$ would receive the same labeled input at level $M$ under both targets. Since both targets must be identified by level $M$, the same output would have to equal both $f$ and $g$, impossible. Therefore $\cX_M$ distinguishes every pair in $\cH$, so
\begin{equation}
  N^{\bar{\cX}}(\cH)\leq M.
\end{equation}
Since this holds for every learner $A$ with a finite uniform identifying level, it gives
\begin{equation}
  N^{\bar{\cX}}(\cH)\leq \min_A\sup_{f\in\cH}N_A^{\bar{\cX}}(f).
\end{equation}

Conversely, suppose $N^{\bar{\cX}}(\cH)=N<\infty$. Define a learner that, on a labeled finite sample, outputs the unique hypothesis in $\cH$ consistent with that sample if such a unique hypothesis exists, and outputs an arbitrary fixed element of $\cH$ otherwise. For every target $f\in\cH$ and every $n\geq N$, the prefix $\cX_n$ contains $\cX_N$, and $\cX_N$ distinguishes every pair in $\cH$. Hence $f$ is the unique hypothesis in $\cH$ consistent with $f|_{\cX_n}$, so the learner outputs $f$. Therefore
\begin{equation}
  \min_A\sup_{f\in\cH}N_A^{\bar{\cX}}(f)\leq N^{\bar{\cX}}(\cH).
\end{equation}
If $N^{\bar{\cX}}(\cH)=\infty$, the first paragraph shows that no learner can have a finite uniform identifying level. Thus both sides are $\infty$. This proves the identity in all cases.
\end{proof}

\begin{proof}[\textbf{Proof of \Cref{thm:target-uniform-equivalences}}]
The implication (i)$\Rightarrow$(ii) is immediate by fixing one hierarchy. If (ii) holds with learner $A$, hierarchy $\bar{\cX}$, and bound $N$, then no two distinct targets can agree on $\cX_N$: otherwise $A$ would receive the same labeled input under both targets at level $N$, but would have to output two different functions. Hence $\cX_N$ distinguishes every pair in $\cH$, so $N^{\bar{\cX}}(\cH)<\infty$, proving (iii).

If (iii) holds, then for some hierarchy $\bar{\cX}$ and some finite $N$, the prefix $\cX_N$ distinguishes every pair in $\cH$. The restriction map $f\mapsto f|_{\cX_N}$ is injective, so $\cH$ injects into the finite set of binary labelings of $\cX_N$. Hence $\cH$ is finite, proving (v).

If $\cH$ is finite, then for every hierarchy $\bar{\cX}$ and every distinct pair $f,g\in\cH$, there is some point $x_{f,g}$ with $f(x_{f,g})\neq g(x_{f,g})$. Since $\bar{\cX}$ exhausts $\Omega$, each such point appears in some finite prefix. Taking the maximum over the finitely many pairs gives a finite level whose prefix distinguishes every pair in $\cH$. Thus (iv) holds. If $\cH$ has at most one element, the distinguishing condition is vacuous and any level, say $1$, works.

Finally, (iv)$\Rightarrow$(iii) is immediate, and (v)$\Rightarrow$(i) also follows from \Cref{thm:A2-finite}. Therefore all five conditions are equivalent.
\end{proof}

\subsection{Domain-Uniform Generalization}
\label{app:A3-A4-proofs}

\begin{proof}[\textbf{Proof of \Cref{thm:A3-brittle}}]
The implication (i)$\Rightarrow$(ii) is immediate, since a single bound that works for all targets and all hierarchies is in particular allowed to depend on the target.

Assume next that $\Omega$ is finite. Under the finite-domain convention, every hierarchy has revealed all of $\Omega$ by level $|\Omega|$. Consider the learner that outputs the unique hypothesis in $\cH$ consistent with the full observed labels once the full domain has been revealed, and is arbitrary before that. For any target $f\in\cH$, the full-domain labeled sample determines $f$ uniquely inside $\cH$. Hence this learner identifies every target along every hierarchy by level $|\Omega|$, proving \Afour, and therefore \Athree.

It remains to prove that \Athree\ implies $\Omega$ is finite, under the assumption $|\cH|>2$. Suppose \Athree\ is witnessed by a learner $A$. For every $f\in\cH$, choose a finite level bound $N_f$ such that
\begin{equation}
  N_A^{\bar{\cX}}(f)\leq N_f
\end{equation}
for every hierarchy $\bar{\cX}$.

We first show that every pair of distinct hypotheses has a finite agreement set. Let $f,g\in\cH$ be distinct and set $M=\max\{N_f,N_g\}$. If $\Eq(f,g)$ contained at least $M$ points, choose a hierarchy whose $M$-th prefix consists of $M$ such agreement points. At level $M$, the labeled samples generated by $f$ and by $g$ are identical. But $M\geq N_f,N_g$, so domain-uniform identifiability requires $A$ to output $f$ on this input under target $f$ and $g$ on the same input under target $g$, impossible. Therefore $\Eq(f,g)$ has fewer than $M$ points, and is finite.

Now choose three distinct hypotheses $h_1,h_2,h_3\in\cH$. Since the labels are binary, at each point $x\in\Omega$ at least two of the three values $h_1(x),h_2(x),h_3(x)$ are equal. Hence
\begin{equation}
  \Omega
  =
  \Eq(h_1,h_2)\cup \Eq(h_1,h_3)\cup \Eq(h_2,h_3).
\end{equation}
Each set on the right is finite by the previous paragraph, so $\Omega$ is finite. This proves (ii)$\Rightarrow$(iii), and completes the equivalence.
\end{proof}

%% file: appendix/app_c_complexity_stratified_proofs.tex
\section{Proofs for Stratified Domain Complexity}
\label{app:complexity-stratified-proofs}

\begin{proof}[\textbf{Proof of \Cref{thm:sublevel-domain-complexity-optimal-rate}}]
Fix a hierarchy $\bar{\cX}$, a class $\cH\subseteq\{0,1\}^{\Omega}$, and a complexity scale $c:\cH\to\mathbb{N}$.

Fix $s\in\mathbb{N}$. By \Cref{eq:finite-sublevel-exhaustion}, $\cH_{\leq s}^{c}$ is finite. For every distinct pair $f,g\in\cH_{\leq s}^{c}$, choose a point $x_{f,g}$ with $f(x_{f,g})\neq g(x_{f,g})$. Since $\bar{\cX}$ exhausts $\Omega$, each such point appears in some finite prefix. Taking the maximum over the finitely many pairs gives a finite prefix that distinguishes every pair in $\cH_{\leq s}^{c}$. If $|\cH_{\leq s}^{c}|\leq1$, any level, say $1$, works. Thus
\begin{equation}
  N_{(\cH,c)}^{\bar{\cX}}(s)<\infty.
\end{equation}

For the optimality identity, fix $s$ with $\cH_{\leq s}^{c}\neq\emptyset$. By definition,
\begin{equation}
  N_{(\cH,c)}^{\bar{\cX}}(s)
  =
  N^{\bar{\cX}}(\cH_{\leq s}^{c}).
\end{equation}
Applying \Cref{prop:class-domain-complexity-optimal} to the class $\cH_{\leq s}^{c}$ gives
\begin{equation}
  N_{(\cH,c)}^{\bar{\cX}}(s)
  =
  \min_A \sup_{f\in\cH_{\leq s}^{c}} N_A^{\bar{\cX}}(f),
\end{equation}
with both sides taking values in $\mathbb{N}_{\geq 1}\cup\{\infty\}$.

Finally suppose $|\cH_{\leq s}^{c}|\to\infty$. The sublevel classes are nested in $s$, so the numbers $N_{(\cH,c)}^{\bar{\cX}}(s)$ are nondecreasing. If they did not tend to infinity, then there would be a finite $M$ and arbitrarily large $s$ with $N_{(\cH,c)}^{\bar{\cX}}(s)\leq M$. For each such $s$, the prefix $\cX_M$ would distinguish every pair in $\cH_{\leq s}^{c}$, so the restriction map into $\{0,1\}^{\cX_M}$ would be injective on $\cH_{\leq s}^{c}$. This gives
\begin{equation}
  |\cH_{\leq s}^{c}|\leq 2^{|\cX_M|}
\end{equation}
for arbitrarily large $s$, contradicting $|\cH_{\leq s}^{c}|\to\infty$. Therefore $N_{(\cH,c)}^{\bar{\cX}}(s)\to\infty$.
\end{proof}

\begin{proof}[\textbf{Proof of \Cref{prop:computability-obstruction}}]
Suppose, toward a contradiction, that for some computable exhaustive hierarchy $\bar{\cX}$ there is a computable function $B$ satisfying the bound in the statement. Given two descriptions $p,q\in\Pi$, compute
\begin{equation}
  s:=\max\{\kappa(p),\kappa(q)\}.
\end{equation}
Compute $B(s)$ and enumerate $\cX_{B(s)}$. Since evaluation is decidable from descriptions, we can check whether $h_p$ and $h_q$ disagree on some point of $\cX_{B(s)}$. If they do, then they are not semantically equivalent. If they do not, then $h_p$ and $h_q$ agree on $\cX_{B(s)}$. But $h_p,h_q\in\cH_{\leq s}^{\Pi,\kappa}$, and $B(s)$ is at least the sublevel domain complexity of this slice. Hence any two distinct hypotheses in $\cH_{\leq s}^{\Pi,\kappa}$ would have been separated inside $\cX_{B(s)}$. Therefore $h_p=h_q$ on all of $\Omega$, and the procedure decides semantic equivalence from descriptions.
\end{proof}

\begin{proof}[\textbf{Proof of \Cref{prop:noncomputable-hierarchy-sublevel-bound}}]
Fix an enumeration $z_1,z_2,\ldots$ of $\Omega$. We construct the hierarchy recursively. Set $\cX_0=\emptyset$. Given $\cX_{m-1}$, for each $i<m$ choose a point $w_{i,m}\in\Omega$ such that $h_i(w_{i,m})\neq h_m(w_{i,m})$. Let $u_m$ be the first point in the enumeration of $\Omega$ that is not contained in
\begin{equation}
  \cX_{m-1}\cup\{w_{i,m}:i<m\}.
\end{equation}
Define
\begin{equation}
  \cX_m:=\cX_{m-1}\cup\{w_{i,m}:i<m\}\cup\{u_m\}.
\end{equation}
Each $\cX_m$ is finite and $\cX_{m-1}\subsetneq\cX_m$. The added points $u_m$ ensure that every point of $\Omega$ eventually appears, so $\bar{\cX}$ is exhaustive. Finally, for any pair $i<j\leq m$, the point $w_{i,j}$ was added by level $j$, and hence belongs to $\cX_m$; it separates $h_i$ from $h_j$. Thus $\cX_m$ distinguishes every pair in $\{h_1,\ldots,h_m\}$, so $N^{\bar{\cX}}(\{h_1,\ldots,h_m\})\leq m$.
\end{proof}

%% file: appendix/app_d_no_free_lunch_proofs.tex
\section{Proofs for the No-Free-Lunch Results}
\label{app:no-free-lunch-proofs}

\subsection{Class-Level No Free Lunch}
\label{app:class-level-nfl-proof}

\begin{proof}[\textbf{Proof of \Cref{thm:domain-generalization-nfl}}]
Choose three distinct hypotheses $h_1,h_2,h_3\in\cH$. Since labels are binary, the three pairwise agreement sets
\begin{equation}
  \Eq(h_1,h_2),\qquad \Eq(h_1,h_3),\qquad \Eq(h_2,h_3)
\end{equation}
cover $\Omega$: at every point, at least two of the three binary labels agree. Since $\Omega$ is infinite, at least one of these three agreement sets is infinite. Let $g_0,g_1\in\cH$ be a distinct pair with $\Eq(g_0,g_1)$ infinite.

Fix $M\in\mathbb{N}_{\geq 1}$. Choose distinct points
\begin{equation}
  z_1,\ldots,z_M\in \Eq(g_0,g_1),
\end{equation}
and extend them to an enumeration $z_1,z_2,\ldots$ of $\Omega$. Define the point-by-point hierarchy
\begin{equation}
  \cX_n:=\{z_1,\ldots,z_n\}.
\end{equation}
Then $g_0$ and $g_1$ agree on $\cX_M$.

Now let $A$ be any learner. At level $M$, the labeled sample generated by target $g_0$ is the same as the labeled sample generated by target $g_1$. Therefore $A$ has the same output on both samples. It cannot be equal to both distinct functions $g_0$ and $g_1$. Hence at least one of the two targets, call it $h$, is not identified by level $M$, so
\begin{equation}
  N_A^{\bar{\cX}}(h)>M.
\end{equation}
The hierarchy was constructed before choosing $A$, so this proves the first claim.

When $\cH$ is finite, the equivalent class-domain-complexity statement follows from \Cref{prop:class-domain-complexity-optimal}. Indeed,
\begin{equation}
  N^{\bar{\cX}}(\cH)>M
\end{equation}
if and only if no learner can identify every target in $\cH$ by level $M$, which is equivalent to saying that every learner has some target $h\in\cH$ with $N_A^{\bar{\cX}}(h)>M$.
\end{proof}

\subsection{Complexity-Stratified No Free Lunch}
\label{app:complexity-stratified-nfl-proof}

\begin{proof}[\textbf{Proof of \Cref{cor:complexity-stratified-nfl}}]
Fix an enumeration
\begin{equation}
  \Omega=\{y_1,y_2,\ldots\}.
\end{equation}
We will construct a point-by-point hierarchy by constructing an enumeration $z_1,z_2,\ldots$ of $\Omega$. The construction ensures that, for infinitely many increasing complexity thresholds, a pair of hypotheses inside the corresponding sublevel class agrees on a very long initial prefix.

\emph{Step 1: a reusable extension claim.}
We first state the combinatorial step used at every stage. Let $P\subseteq\Omega$ be finite, let $S\in\mathbb{N}$, and let $R\subseteq\Omega$ be infinite. Since $\cH_{\leq S}^{c}$ is finite and $\cH$ is infinite, the set $\cH\setminus\cH_{\leq S}^{c}$ is infinite. Only finitely many binary label patterns can occur on $P$, so infinitely many hypotheses in $\cH\setminus\cH_{\leq S}^{c}$ share the same pattern on $P$. Choose three such hypotheses. On every point of $R$, at least two of their three binary labels agree; since the three pairwise agreement sets cover $R$, one pair agrees on infinitely many points of $R$. Thus there exist distinct $f,g\in\cH\setminus\cH_{\leq S}^{c}$ such that
\begin{equation}
  f|_P=g|_P
  \qquad\text{and}\qquad
  \Eq(f,g)\cap R \text{ is infinite}.
\end{equation}

\emph{Step 2: build a hierarchy with longer and longer concealed pairs.}
Choose any three hypotheses in $\cH$. Their three pairwise agreement sets cover the infinite set $\Omega$, so one pair $(f_0,g_0)$ has infinite agreement set. Let
\begin{equation}
  A_0:=\Eq(f_0,g_0),
  \qquad
  S_0:=\max\{c(f_0),c(g_0)\},
  \qquad
  P_0:=\emptyset.
\end{equation}
Inductively, suppose we have a finite prefix $P_k$ of the ordering, a pair $(f_k,g_k)$, an infinite agreement set $A_k:=\Eq(f_k,g_k)$, and a threshold $S_k:=\max\{c(f_k),c(g_k)\}$, with $P_k\subseteq A_k$. The invariant means that the current pair agrees on every point revealed so far.

Apply the combinatorial step with the finite set $P_k\cup\{y_{k+1}\}$, the threshold $S_k$, and the infinite set $A_k$. We obtain a pair $(f_{k+1},g_{k+1})$ outside $\cH_{\leq S_k}^{c}$ such that it agrees on $P_k\cup\{y_{k+1}\}$ and has infinitely many agreement points inside $A_k$. Set
\begin{equation}
  A_{k+1}:=\Eq(f_{k+1},g_{k+1}),
  \qquad
  S_{k+1}:=\max\{c(f_{k+1}),c(g_{k+1})\}.
\end{equation}
Because both $f_{k+1}$ and $g_{k+1}$ lie outside $\cH_{\leq S_k}^{c}$, we have $S_{k+1}>S_k$. Also $A_k\cap A_{k+1}$ is infinite, and $P_k\cup\{y_{k+1}\}\subseteq A_{k+1}$.

Extend the current ordering by adding unused points from $A_k\cap A_{k+1}$ until at least the first $B(S_{k+1})$ positions have been filled. Then append $y_{k+1}$ if it has not already appeared. Let $P_{k+1}$ be the resulting finite prefix. This preserves the induction invariant $P_{k+1}\subseteq A_{k+1}$: the old prefix and $y_{k+1}$ lie in $A_{k+1}$, and the added points lie in $A_k\cap A_{k+1}$. It also ensures that the first $B(S_{k+1})$ points of the ordering lie in $A_k$.

Since $y_{k+1}$ is inserted by the end of stage $k+1$ if it has not appeared earlier, every point of $\Omega$ eventually appears. The limiting ordering $z_1,z_2,\ldots$ is therefore an enumeration of $\Omega$, and it defines a point-by-point hierarchy $\bar{\cX}$ by
\begin{equation}
  \cX_n:=\{z_1,\ldots,z_n\}.
\end{equation}

\emph{Step 3: verify the lower bound on the sublevel domain complexity.}
It remains to verify the bound. The thresholds $S_k$ are strictly increasing, hence $S_k\to\infty$. Let $s\geq S_0$. Choose $k$ such that
\begin{equation}
  S_k\leq s\leq S_{k+1}.
\end{equation}
The pair $(f_k,g_k)$ lies in $\cH_{\leq s}^{c}$ because both complexities are at most $S_k\leq s$. By construction, $f_k$ and $g_k$ agree on the first $B(S_{k+1})$ points of the hierarchy. Since $B$ is nondecreasing and $s\leq S_{k+1}$, they agree on the first $B(s)$ points. Therefore $\cX_{B(s)}$ does not distinguish all pairs in $\cH_{\leq s}^{c}$, and by the definition of sublevel domain complexity,
\begin{equation}
  N_{(\cH,c)}^{\bar{\cX}}(s)>B(s).
\end{equation}
Thus the corollary holds with $s_0=S_0$.
\end{proof}

%% file: appendix/app_e_online_learning.tex
\section{Proofs for Online Learning}
\label{app:online-learning-proofs}

\begin{proof}[\textbf{Proof of \Cref{thm:Ldim-mistake-bound}}]
We first prove the upper bound when $d:=\Ldim(\cH)<\infty$, using the convention $\Ldim(\emptyset)=-1$. The learner maintains the version space $V_n\subseteq\cH$ of hypotheses consistent with all labels revealed before level $n$. Thus $V_1=\cH$, and after level $n$ is revealed the version space is restricted to hypotheses consistent with the new labels.

For a point $x$ and a label $b\in\{0,1\}$, write
\begin{equation}
  V_n(x,b):=\{h\in V_n:h(x)=b\}.
\end{equation}
Before seeing the labels on $\cX_n\setminus\cX_{n-1}$, the learner predicts at each $x$ a label $b$ maximizing $\Ldim(V_n(x,b))$, breaking ties arbitrarily. After all labels at level $n$ are revealed, it updates the version space to the subset of $V_n$ consistent with the whole increment.

We claim that every mistaken level strictly decreases the Littlestone dimension of the version space. Suppose level $n$ is mistaken. Then for some $x\in\cX_n\setminus\cX_{n-1}$, the learner predicted $1-b$ while the true label is $b=f(x)$. Since the learner chose a label of maximum Littlestone dimension,
\begin{equation}
  \Ldim(V_n(x,1-b))\geq \Ldim(V_n(x,b)).
\end{equation}
Both $V_n(x,0)$ and $V_n(x,1)$ are subclasses of $V_n$, so their Littlestone dimensions are at most $\Ldim(V_n)$. If $\Ldim(V_n(x,b))=\Ldim(V_n)$, then also $\Ldim(V_n(x,1-b))=\Ldim(V_n)$, and placing $x$ at the root above two shattered trees of depth $\Ldim(V_n)$ would shatter a tree of depth $\Ldim(V_n)+1$, a contradiction. Hence
\begin{equation}
  \Ldim(V_n(x,b))<\Ldim(V_n).
\end{equation}
The updated version space after the whole increment is contained in $V_n(x,b)$, so its Littlestone dimension is also strictly smaller than $\Ldim(V_n)$. Since the dimension starts at $d$ and never goes below zero, at most $d$ levels can be mistaken.

For the lower bound, let $d$ be any integer such that $\cH$ shatters a complete binary tree of depth $d$. We show that every deterministic learner can be forced to make $d$ mistakes; if $\Ldim(\cH)=\infty$, this holds for every finite $d$.

The adversary follows a shattered tree. At the root, it presents the root instance as the first singleton increment. If the learner predicts $0$, the adversary reveals label $1$; if the learner predicts $1$, it reveals label $0$. The adversary then moves to the child corresponding to the revealed label and repeats this procedure for $d$ levels. By shattering, the resulting root-to-leaf sequence of labels is realized by some target $f\in\cH$.

The points along the realized path are distinct. If the same point appeared twice on one root-to-leaf path, then at the later occurrence the two outgoing branches would require both labels while the earlier occurrence has already fixed one label, contradicting shattering. Thus the path points can be used as distinct singleton increments of a hierarchy, which is then completed arbitrarily to exhaust $\Omega$. By construction, the learner makes one mistake at each of the first $d$ levels. Therefore no learner can have a uniform mistake bound smaller than $d$. Taking the supremum over all shattered depths gives the lower bound $\Ldim(\cH)$.
\end{proof}

\begin{proof}[\textbf{Proof of \Cref{prop:mistake-intermediate}}]
If $\cH$ shatters a complete binary tree of depth $d$, then the $2^d$ root-to-leaf label sequences must be realized by $2^d$ distinct hypotheses: two paths that first diverge at a node assign different labels to the same instance at that node. Hence, when $\cH$ is finite,
\begin{equation}
  2^d\leq |\cH|
\end{equation}
for every shattered depth $d$, so $\Ldim(\cH)\leq \log_2|\cH|<\infty$.

Now suppose $\Ldim(\cH)=d<\infty$. By \Cref{thm:Ldim-mistake-bound}, there is a deterministic online prediction rule that makes at most $d$ mistakes on every target and hierarchy. Fix a singleton hierarchy enumerating $\Omega$, say $\cX_n=\{x_1,\ldots,x_n\}$.

For each target $f\in\cH$, let $E_f\subseteq\mathbb{N}$ be the finite set of times at which this learner makes a mistake when run against $f$ on the fixed hierarchy. Then $|E_f|\leq d$. The set $E_f$ determines $f$: given $E_f$, replay the learner from the beginning; at time $n$, compute its prediction from the labels reconstructed so far, and set the true label to be the prediction if $n\notin E_f$ and the opposite label if $n\in E_f$. This reconstructs $f(x_n)$ for every $n$. Since the hierarchy enumerates $\Omega$, it reconstructs $f$ on all of $\Omega$.

Thus the map $f\mapsto E_f$ is injective from $\cH$ into the collection of finite subsets of $\mathbb{N}$ of size at most $d$, which is countable. Therefore $\cH$ is countable.
\end{proof}

\begin{proof}[Verification of \Cref{ex:singleton-ldim-one}]
The class shatters a one-point tree: at any point $x_j$, the hypothesis $h_j$ realizes label $1$, while any $h_i$ with $i\neq j$ realizes label $0$. Hence $\Ldim(\cH_{\mathrm{sing}})\geq 1$.

It cannot shatter a depth-two tree. If the root is labeled by some point $x_j$ and the root label is $1$, then the only consistent target is $h_j$. All later labels are then fixed, so the next node cannot realize both outgoing labels. Thus no depth-two tree is shattered, and $\Ldim(\cH_{\mathrm{sing}})=1$.

Equivalently, the all-zero online predictor makes at most one mistake on this class, namely when the unique positive point of the target singleton is revealed.

Finally, consider the singleton hierarchy $\cX_n=\{x_1,\ldots,x_n\}$. Suppose \Atwo\ held along this hierarchy with some learner and common level $N$. For every $i>N$, the target $h_i$ gives the same all-zero labeled sample on $\cX_N$. At level $N$, the learner would therefore receive the same input for all targets $h_i$ with $i>N$, but would have to output each of these distinct targets. This is impossible. Hence no target-uniform identification bound exists, so \Atwo\ fails.
\end{proof}

%% file: appendix/app_f_approximation.tex
\section{Proofs for Approximation}
\label{app:approximation-proofs}

\begin{proof}[\textbf{Proof of \Cref{prop:approximate-hierarchy-nfl}}]
Choose three distinct hypotheses in $\cH$. Since the labels are binary, the three pairwise agreement sets cover $\Omega$; because $\Omega$ is infinite, some pair $f,g$ agrees on infinitely many points. Since $P$ has full support and $f\neq g$,
\begin{equation}
  \Delta:=P(\{f\neq g\})>0.
\end{equation}
Set $\varepsilon:=\Delta/3$. Given $M$, choose a finite set $S\subseteq\{f=g\}$ of size $M$, list these points first, and complete the list to a hierarchy with $\cX_M=S$. Since $f|_S=g|_S$, the learner returns the same predictor on the two training sets; call it $\hat h$. On every point where $f$ and $g$ disagree, $\hat h$ is wrong for at least one of the two targets. Hence
\begin{equation}
  \Delta
  \leq
  L_{P,f}(\hat h)+L_{P,g}(\hat h).
\end{equation}
Therefore at least one of the two targets, say $u\in\{f,g\}$, has $L_{P,u}(A(u|_{\cX_M}))\geq\Delta/2>\varepsilon$, and so $N_{A,\varepsilon}^{\bar{\cX},P}(u)>M$.
\end{proof}

%% file: appendix/app_f_out_of_support.tex
\section{Proofs for Out-of-Support Generalization}
\label{app:out-of-support-proofs}

\begin{proof}[\textbf{Proof of \Cref{thm:full-support-source-main}}]
Enumerate the finite class as $\cH=\{h_1,\ldots,h_M\}$ and use the least-index consistent learner: given a finite labeled sample, output the first hypothesis in the enumeration consistent with all observed labels, if one exists, and output $h_1$ otherwise. The learner ignores the order and multiplicity of repeated sample points.

If $M=1$, the claim is immediate. Assume $M\geq2$. For every pair of distinct hypotheses $h_j,h_k\in\cH$, define
\begin{equation}
  D_{j,k}:=\{x\in\Omega:h_j(x)\neq h_k(x)\}.
\end{equation}
Each $D_{j,k}$ is nonempty and has positive $Q$-mass because $Q$ has full support. Let
\begin{equation}
  \alpha:=\min_{j\neq k} Q(D_{j,k})>0.
\end{equation}
If the sample contains at least one point from every pairwise disagreement set, then the observed labels distinguish the target from every other hypothesis in $\cH$. The least-index consistent learner therefore outputs the target, and the $P$-error is zero.

By a union bound, the probability of missing some pairwise disagreement set is at most
\begin{equation}
  M(M-1)(1-\alpha)^n,
\end{equation}
which tends to zero as $n\to\infty$. Hence for every $\delta>0$ there is a finite $N$ such that, for all $n\geq N$, the sample hits every pairwise disagreement set with probability at least $1-\delta$. On this event, the learner outputs the target for every $f\in\cH$, and therefore has $P$-error at most $\varepsilon$ for every target. This gives the target-uniform bound.
\end{proof}

\begin{proof}[\textbf{Proof of \Cref{thm:infinite-support-source-failure}}]
Choose three distinct hypotheses in $\cH$. Their three pairwise agreement sets cover the infinite domain $\Omega$, so some pair $u,v\in\cH$ has an infinite agreement set $\Eq(u,v)$. Choose any distribution $Q$ with infinite support contained in $\Eq(u,v)$; for example, enumerate countably many distinct points in $\Eq(u,v)$ and assign them probabilities proportional to $2^{-k}$.

Since $u\neq v$, the disagreement set $\{u\neq v\}$ is nonempty. Because $P$ has full support,
\begin{equation}
  \Delta:=P(\{u\neq v\})>0.
\end{equation}
Choose $\varepsilon<\Delta/2$.

Fix a learner $A$ and a sample size $n$.
Samples from $Q$ always lie in $\Eq(u,v)$, so the labeled sample given to $A$ is identical under target $u$ and target $v$. For any fixed sample sequence, let $\hat h$ be this common learner output. At each point where $u$ and $v$ disagree, $\hat h$ is wrong for at least one of the two targets. Therefore
\begin{equation}
  L_{P,u}(\hat h)+L_{P,v}(\hat h)\geq \Delta.
\end{equation}
Thus, for every sample sequence, at least one of the two targets has error larger than $\varepsilon$. Averaging over $Q^n$, at least one of the two targets has failure probability at least $1/2$. Since $A$ and $n$ were arbitrary, this holds for every learner and every sample size; in particular, it holds for every $N$ and every $n\geq N$.
\end{proof}

%% file: appendix/app_h_positive_only.tex
\section{Proofs for Positive-Only Presentations}
\label{app:positive-only-proofs}

\begin{proof}[\textbf{Proof of \Cref{prop:positive-only-uniform-identification}}]
First suppose
\begin{equation}
  d:=\sup_{\substack{f,g\in\cH\\ f\neq g}}|\mathrm{Pos}(f)\cap \mathrm{Pos}(g)|<\infty.
\end{equation}
Define a learner as follows. Given a finite positive set $S$, if there is a unique $h\in\cH$ such that $S\subseteq \mathrm{Pos}(h)$, output that $h$; otherwise output an arbitrary binary predictor.

Fix a target $f\in\cH$ and a positive presentation $\bar S\in\mathrm{Pres}^{+}(f)$. Since the presentation is increasing and every step adds a positive point, $|S_t|\geq t$. Hence $|S_t|\geq d+1$ for every $t\geq d+1$. No $g\neq f$ can contain all points in $S_t$, because then $|\mathrm{Pos}(f)\cap \mathrm{Pos}(g)|\geq |S_t|\geq d+1$, contradicting the definition of $d$. Thus $f$ is the unique hypothesis whose positive region contains $S_t$, and the learner outputs $f$ for all $t\geq d+1$.

Conversely, suppose
\begin{equation}
  \sup_{\substack{f,g\in\cH\\ f\neq g}}|\mathrm{Pos}(f)\cap \mathrm{Pos}(g)|=\infty.
\end{equation}
Let $A$ be any learner and let $N<\infty$. Choose $f\neq g$ with $|\mathrm{Pos}(f)\cap \mathrm{Pos}(g)|\geq N$, and choose distinct points $x_1,\ldots,x_N\in \mathrm{Pos}(f)\cap \mathrm{Pos}(g)$. Since both positive regions are infinite, we can extend the common initial sequence
\begin{equation}
  S_t=\{x_1,\ldots,x_t\},
  \qquad 1\leq t\leq N,
\end{equation}
to positive presentations of both $f$ and $g$. At time $N$, the learner receives the same input under the two targets. It therefore gives the same output in both worlds, and cannot be correct for both $f$ and $g$. No learner can satisfy \Cref{eq:positive-only-uniform-id} with this $N$, and since $N$ was arbitrary, no uniform finite bound exists.
\end{proof}

\paragraph{An infinite example.}
The finite-overlap condition can hold for infinite classes. Let $\Omega=\mathbb{N}\times\mathbb{N}$ and, for $a,b\in\mathbb{N}$, define
\begin{equation}
  G_{a,b}:=\{(t,at+b):t\in\mathbb{N}\},
\end{equation}
and let $f_{a,b}$ be the indicator of $G_{a,b}$. Each $\mathrm{Pos}(f_{a,b})=G_{a,b}$ is infinite, and the class $\{f_{a,b}:a,b\in\mathbb{N}\}$ is countably infinite. Two distinct such sets intersect in at most one point, since two distinct affine functions over $\mathbb{N}$ agree at at most one input. Hence the pairwise positive overlap is at most one, and two positive observations uniformly identify the target.

%% file: appendix/app_g_limitations.tex
\section{Limitations and Scope}
\label{app:limitations}

Our results are information-theoretic: they characterize when identification is possible, not how to compute the identifying level for a given $(\cH, c, \bar{\cX})$. Three structural restrictions are worth noting. First, we work in the realizable, deterministic-labeling setting; agnostic and noisy presentations are deferred to future work. Second, we focus on exact identification; approximate analogues with $\varepsilon$-error are discussed in \Cref{sec:approximation}. Third, the framework is set-indexed rather than sequence-indexed, hiding within-level order; online-learning variants are discussed in \Cref{sec:online-learning}. The work is theoretical and does not raise direct societal concerns.